\documentclass[letterpaper, 10 pt, journal]{IEEEtran}

\usepackage{graphicx}
\usepackage{epstopdf}
\usepackage{amsmath}
\usepackage{amsthm}
\usepackage{amssymb}
\usepackage{multirow,booktabs}
\usepackage{xcolor}

\newcommand{\bbR}{\ensuremath{\mathbb{R}}}

\newcommand{\bbN}{\ensuremath{\mathbb{N}}}

\newcommand{\bigO}[1]{\ensuremath{\mathcal{O}\left( #1 \right)}}
\newcommand{\mcN}{\ensuremath{\mathcal{N}}}

\newcommand{\satThresh}{\ensuremath{\mathcal{M}}}

\newcommand{\bfm}{\ensuremath{\mathbf{m}}}

\newcommand{\bfx}{\ensuremath{\mathbf{x}}}

\newcommand{\bfmu}{\ensuremath{\boldsymbol \mu}}

\newcommand{\until}[1]{\{1,\dots, #1\}}

\newcommand{\supscr}[2]{#1^{\textup{#2}}}

\newcommand{\seqdef}[2]{\{#1\}_{#2}}

\newcommand{\indicator}[1]{\ensuremath{\mathbf{1}\! \left( #1 \right)}}
\renewcommand{\Pr}[1]{\ensuremath{\mathrm{Pr} \left[ #1 \right]}}
\newcommand{\E}[1]{\ensuremath{\mathbb{E}\left[ #1 \right]}}

\DeclareMathOperator*{\argmax}{arg\,max}

\def \etal {\emph{et al.}}

\newcommand{\beq}{\begin{equation}}
\newcommand{\eeq}{\end{equation}}
\newcommand{\bal}{\begin{align}}
\newcommand{\eal}{\end{align}}

\newcommand{\red}[1]{{\color{red} #1}}

\newtheorem{theorem}{Theorem}
\newtheorem{lemma}[theorem]{Lemma}
\newtheorem{remark}[theorem]{Remark}
\newtheorem{corollary}[theorem]{Corollary}

\newtheorem{defn}{Definition}

\newcommand\bit[1]{\textit{\textbf{#1}}}

\usepackage{soul}

\setstcolor{red}

\title{\LARGE \bf
Satisficing in multi-armed bandit problems
\thanks{This research has been supported in part by ONR grant N00014-14-1-0635 and ARO grant W911NF-14-1-0431. P. Reverdy was supported through a NDSEG Fellowship. P. Reverdy is with the Department of Electrical and Systems Engineering, University of Pennsylvania, Philadelphia, PA 19104, USA {\tt \small preverdy@seas.upenn.edu}. V. Srivastava is with the Department of Electrical and Computer Engineering, Michigan State University, East Lansing, MI 48824, USA  {\tt\small vaibhav@egr.msu.edu}.
N. E. Leonard is with the Department of Mechanical and Aerospace Engineering, Princeton University, Princeton, NJ 08544, USA
        {\tt\small naomi@princeton.edu}}%
}

\author{Paul Reverdy, Vaibhav Srivastava, and Naomi Ehrich Leonard
}

\begin{document}

\maketitle

\begin{abstract}
Satisficing is a relaxation of maximizing and allows for less risky decision making in the face of uncertainty. We propose two sets of satisficing objectives for the multi-armed bandit problem, where the objective is to achieve reward-based decision-making performance above a given threshold. We show that these new problems are equivalent to various standard multi-armed bandit problems with maximizing objectives and use the equivalence to find bounds on performance. The different objectives can result in qualitatively different behavior; for example, agents explore their options continually in one case and only a finite number of times in another. For the case of Gaussian rewards we show an additional equivalence between the two sets of satisficing objectives that allows algorithms developed for one set to be applied to the other. We then develop variants of the Upper Credible Limit (UCL) algorithm that solve the problems with satisficing objectives and show that these modified UCL algorithms achieve efficient satisficing performance.
\end{abstract}

\section{Introduction}
Engineering solutions to decision-making problems are often designed to maximize an objective function. However, in many contexts maximization of an objective function is an unreasonable goal, either because the objective itself is poorly defined or because solving the resulting optimization problem is intractable or costly. In these contexts, it is valuable to consider alternative decision-making frameworks.

Herbert Simon considered alternative models of rational decision-making~\cite{HAS:55} with the goal of making them ``compatible with the access to information and the computational capacities that are actually possessed by organisms, including man, in the kinds of environments in which such organisms exist.'' A major feature of the models he considered is what he called ``satisficing''. In \cite{HAS:55}, he discussed in very broad terms a variety of simplifications to the classical economic concept of rationality, most importantly the idea that payoffs should be simple, defined by doing well relative to some threshold value. In \cite{HAS:56}, he introduced the word ``satisficing'', a combination of the words ``satisfy'' and ``suffice'', to refer to this thresholding concept and illustrated it using a mathematical model of foraging. He also briefly discussed how satisficing relates to problems in inventory control and more complicated decision processes like playing chess.

Since Simon's pioneering work, satisficing has been studied in many fields such as psychology~\cite{BS-etal:02}, economics~\cite{RB-ML:00}, management science~\cite{TMM:84,SGW:00}, and ecology~\cite{DW:92,YC-YBH:05}. In engineering, satisficing is of interest for the same reasons that motivated its introduction in the social science literature, specifically that it can simplify decision-making problems: as compared to maximizing it allows for less risky decision making in the face of uncertainty.
Furthermore, many engineering problems are naturally posed using a satisficing objective, such as choosing a design that meets given specifications, but where the designers may be indifferent among any such designs. Satisficing is well defined even if there are several competing performance measures that trade off in complicated ways, whereas maximizing may be poorly defined without additional information about preferences.

Satisficing has been studied in the engineering literature in several contexts. In~\cite{HN:84}, the authors studied design optimization using a satisficing objective and found that it is effective in many practical fields. In~\cite{MAG-WCS-RLF:98}, the authors studied control theory using a satisficing objective function, and in~\cite{BY-etal:13}, the authors used satisficing to study optimal software design. In~\cite{AC-etal:15}, the authors used a multi-armed bandit algorithm to construct robots that actively adapt their control policies to mitigate damage, such as actuator failures. In order to speed the convergence of their algorithm, they only sought to identify control policies with performance above a set threshold, rather than to identify an optimal policy. The theory that we develop in this paper formalizes their notion of thresholding and provides bounds on performance.

In this paper, we consider satisficing in the stochastic multi-armed bandit problem \cite{HR:52}, for which a decision maker sequentially chooses one of a set of alternative options, called arms, and earns a reward drawn from a stationary probability distribution associated with that arm. The standard multi-armed bandit problem uses a maximizing objective on accumulated reward. For this objective there is a known performance bound in terms of expected \emph{regret}, which is the expected difference between the reward received by the decision maker and the maximum reward possible.

Since the standard notion of regret is defined relative to the unknown optimum, it can only be computed by an omniscient agent; this notion of regret is not computable by a decision maker faced with a multi-armed bandit problem. Nevertheless, it is a useful theoretical concept, which facilitates the analysis of algorithms designed to solve bandit problems. We extend the notion of regret to satisficing objectives and use it to analyze new algorithms.

In contrast to the standard stochastic multi-armed bandit problem in which the agent seeks to determine, with certainty, the option with maximum mean reward, the  satisficing multi-armed bandit problem seeks to determine, with a desired confidence, a satisfying option. We characterize satisficing in multi-armed bandit problems using three separate features of the satisficing objective.

The first feature selects the quantity on which the satisficing objective is defined. 
We consider two such quantities: (i) the unknown mean reward of the selected option, and (ii) the instantaneous observed reward.

The second feature treats the satisfaction aspect of the satisficing problem. In particular, it selects if the objective function should be optimizing, or if it should be satisfying.

The third feature treats the sufficing aspect of the satisficing problem. In particular, it selects if the decision-making algorithm should be certain that the optimizing/satisfying criterion is met, or if it is sufficient for the algorithm to meet a desired threshold in confidence about the criterion.  Different combinations of the above three features of satisficing lead to eight satisficing objectives that we discuss in this paper.

We begin by defining the four objectives for the case where the satisficing quantity is the unknown mean reward. We show that the bandit problem with each of these four objectives is equivalent to a previously-studied bandit problem and use the equivalence to derive a performance bound for the satisficing problems. These four objectives seek an arm with satisfyingly high mean reward without regard to that reward's dispersion. To develop objectives with improved robustness properties, we then consider the case where the satisficing quantity is the instantaneous observed reward. We extend the first four objectives to this case by adding an additional layer of thresholding, which defines four more objectives. When the reward distributions belong to location-scale families, there is an equivalence between the objectives defined in terms of mean reward and the robust objectives defined in terms of instantaneous reward, which we prove \red{for} Gaussian rewards.

For simplicity of exposition, we then specialize to Gaussian multi-armed bandit problems, where the reward distributions are Gaussian with unknown mean and known variance. For such problems, we develop several modifications of the UCL algorithm that we developed in previous work \cite{PR-VS-NEL:14}. These algorithms solve the problem with the satisficing mean reward objectives (and thus also with the robust objectives); and we show that these algorithms achieve efficient performance. These results extend our previous work \cite{PR-NEL:14a} by incorporating the concept of sufficiency into the satisficing objective, as well as adding several new algorithms and their associated analysis.

The assumption of Gaussian rewards with known variance is not required, but it allows us to focus on the different notions of regret, which is the main contribution of this paper.  We later show how the known variance assumption can be relaxed.  Our methods also extend immediately to many other important classes of reward distributions, including distributions with bounded support and sub-Gaussian distributions. We show how to extend our methods in these cases and provide references to the relevant literature for other extensions.

The remainder of the paper is structured as follows. In Section II we review the standard stochastic multi-armed bandit problem and the associated performance bounds. In Section III we propose the satisficing objectives and bound performance in terms of these objectives. In Section IV we specialize to the case of Gaussian rewards and show the equivalence between the satisficing in mean reward objectives and the satisficing in instantaneous observed reward objectives. In Section V we review the UCL algorithm, and in Section VI we design modified versions of the UCL algorithm for the satisficing objectives. We show that these modified algorithms achieve efficient performance for Gaussian rewards.
We show the results of numerical simulations in Section VII  and in Section VIII we conclude.

\section{The stochastic multi-armed bandit problem}
In the stochastic multi-armed bandit problem a decision-making agent sequentially chooses one among a set of $N$ options called \emph{arms} in analogy with the lever of a slot machine. A single-levered slot machine is called a \emph{one-armed bandit}, so the case of $N \geq 2$ options is called a \emph{multi-armed bandit}.

The decision-making agent collects reward $r_t \in \bbR$ by choosing arm $i_t$ at each time $t \in \{1, \ldots, T\}$, where $T \in \bbN$ is the horizon length for the sequential decision process. The reward from option $i \in \{1, \ldots, N\}$ is sampled from a stationary probability distribution $\nu_i$ and has an unknown mean $m_i \in \bbR$. The decision-maker's objective is to maximize some function of the sequence of rewards $\{r_t\}$ by sequentially picking arms $i_t$ using only the information available at time $t$.

\subsection{Maximization objective}
In the standard multi-armed bandit problem, the agent's objective is to maximize the expected cumulative reward 
\beq \label{eq:maxObjective}
J = \E{\sum_{t=1}^T r_t} = \sum_{t=1}^T m_{i_t}.
\eeq
Equivalently, by defining $m_{i^*} = \max_i m_i$ and $R_t = m_{i^*} - m_{i_t}$,  expected \emph{regret} at time $t$, minimizing \eqref{eq:maxObjective} can be formulated as minimizing the cumulative expected regret defined by
\beq \label{eq:regretObjective}
\sum_{t=1}^T R_t = Tm_{i^*} - \sum_{i=1}^N m_i \E{n_i^T} = \sum_{i=1}^N \Delta_i \E{n_i^T},
\eeq
where $n_i^T$ is the number of times arm $i$ has been chosen up to time $T$, $\Delta_i = m_{i^*}-m_i$ is the expected regret due to picking arm $i$ instead of arm $i^*$, and the expectation is over the possible rewards and decisions made by the agent.

The interpretation of \eqref{eq:regretObjective} is that suboptimal arms $i \neq i^*$ should be chosen as rarely as possible. This is a non-trivial task since the mean rewards $m_i$ are initially unknown to the decision-maker, who must try arms to learn about their rewards while preferentially picking arms that appear more rewarding. The tension between these requirements is known as the \emph{explore-exploit} tradeoff and is common to many problems in machine learning and adaptive control.

\subsection{Bound on optimal performance}
Optimal performance in a bandit problem corresponds to picking suboptimal arms as rarely as possible, as shown by  \eqref{eq:regretObjective}. Lai and Robbins~\cite{TLL-HR:85} studied the standard stochastic multi-armed bandit problem and showed that any policy solving the problem must pick each suboptimal arm $i \neq i^*$ a number of times that is at least logarithmic in the time horizon $T$, i.e.,
\beq \label{eq:LaiRobbinsBound}
\E{n_i^T} \geq \left( \frac{1}{D(\nu_i || \nu_{i^*})} + o(1) \right) \log T,
\eeq
where $o(1) \to 0$ as $T \to +\infty$. The quantity $D(\nu_i || \nu_{i^*}) := \int \nu_i(r) \log \frac{\nu_i(r)}{\nu_{i^*}(r)} \mathrm{d}r$ is the Kullback-Leibler divergence between the reward density $\nu_i$ of any suboptimal arm and the reward density $\nu_{i^*}$ of the optimal arm. Equation \eqref{eq:LaiRobbinsBound} implies that cumulative expected regret must grow at least logarithmically in time.

The bound \eqref{eq:LaiRobbinsBound} is asymptotic in time, but researchers (e.g., \cite{PA-NCB-PF:02,AG-OC:11,PR-VS-NEL:14}) have developed algorithms that achieve cumulative expected regret that is bounded by a logarithmic term uniformly in time, sometimes with the same constant as in \eqref{eq:LaiRobbinsBound}. Cumulative expected regret that is uniformly bounded in time by a logarithmic term is often called \emph{logarithmic regret} for short. In the literature, algorithms that achieve logarithmic regret with a leading term that is within a constant factor of that in \eqref{eq:LaiRobbinsBound} are considered to have optimal performance.

\subsection{Multiple plays}
Anantharam~\etal~\cite{VA-PV-JW:87} studied a generalization of the multi-armed bandit problem in which the agent picks $k \geq 1$ arms at each time $t$, which they called the multi-armed bandit problem with multiple plays. The case $k= 1$ corresponds to the standard multi-armed bandit problem defined above.

In the spirit of \cite{VA-PV-JW:87}, let $\sigma$ be a permutation of $\{1, \ldots, N\}$ such that $m_{\sigma(1)} \geq m_{\sigma(2)} \geq \cdots m_{\sigma(N)}.$ For the multi-armed bandit problem with $k$ plays, the optimal policy with full information corresponds to picking the arms $\sigma(1), \cdots, \sigma(k)$, called the \emph{$k$-best} arms~\cite{VA-PV-JW:87}. In the case $k=1$, $\sigma(1) = i^*$, the optimal arm defined above. For the case of general $k \geq 1$, the cumulative expected regret for the multi-armed bandit problem with multiple plays is defined as follows~\cite{VA-PV-JW:87}:
\beq \label{eq:regretMultiplePlays}
T \sum_{i=1}^k m_{\sigma(i)} - \sum_{i=1}^N m_i \E{n_i^T},
\eeq
which is a straightforward generalization of the regret \eqref{eq:regretObjective}. The suboptimal arms $\sigma(k+1), \cdots, \sigma(N)$ are called the \emph{$k$-worst} arms~\cite{VA-PV-JW:87}. Define $\Delta^{(k)}_i = m_{\sigma(k)} - m_i$ for each $k$-worst arm $i$. The quantity $\Delta^{(k)}_i$ is the generalization of the expected regret $\Delta_i$ for the problem with multiple plays, where the expected value of the optimal policy is that of the $k$ best arms.

As in the case of a single play, optimal performance corresponds to picking suboptimal (i.e., $k$-worst) arms as rarely as possible. By~\cite{VA-PV-JW:87}  each $k$-worst arm $i$ must be picked a number of times that is at least logarithmic in the time horizon $T$, i.e.,
\beq \label{eq:regretBoundMultiplePlays}
\E{n_i^T} \geq \left( \frac{1}{D( \nu_i || \nu_{\sigma(k)} )} + o(1) \right) \log T.
\eeq
This bound can be interpreted as a generalization of the Lai-Robbins bound \eqref{eq:LaiRobbinsBound} where the Kullback-Leibler divergence is taken with respect to the $k^{th}$ best arm $\sigma(k)$ rather than the first best arm $\sigma(1)$ (i.e., $i^*$ in the case $k=1$).

\subsection{PAC bounds}
In the standard multi-armed bandit problem and the multi-armed bandit problem with multiple plays, regret is defined in terms of the unknown mean reward values $m_i$. These regret definitions imply that avoiding regret requires identifying optimal arms with certainty. The requirement to identify optimal arms with certainty is characteristic of a maximizing decision-making strategy. In contrast, a satisficing decision-making agent should seek arms that are ``good enough''. In this context, satisficing corresponds to finding arms that are optimal with high probability rather than with certainty.

The Probably Approximately Correct (PAC) model for learning introduced by Valiant \cite{LGV:84} provides a natural way to capture this aspect of satisficing. Even-Dar \etal~\cite{EED-SM-YM:02,EED-SM-YM:06} and Mannor and Tsitsiklis \cite{SM-JNT:04} studied the multi-armed bandit problem using the PAC model and defined an \emph{$\epsilon$-optimal arm} $i$ as one for which $m_i > m_{i^*} - \epsilon$, i.e., the mean reward is within $\epsilon$ of the optimum value. Equivalently, an $\epsilon$-optimal arm is an arm $i$ for which the expected regret $\Delta_i$ is at most $\epsilon$. Under the PAC model one wishes to find an $\epsilon$-optimal arm with probability of at least $1-\delta$. With probability one, this can be achieved in a finite number of samples, so performance guarantees take the form of bounds on the number of samples required, which is referred to as \emph{sample complexity}. In our notation, we denote sample complexity by $T^*$, as it is the value of the horizon length at which sampling terminates.

When the rewards are Bernoulli distributed with unknown success probabilities $p_i$, the following lower bound holds~\cite{SM-JNT:04}:
\beq \label{eq:explore1Bound}
\E{T^*} \geq \bigO{\frac{1}{\epsilon^2} \log (1/\delta)}.
\eeq
A similar result  was reported in~\cite{EED-SM-YM:02} for $T^*$, rather than its expected value. In other words, one must sample an arm at least $\log(1/\delta)/\epsilon^2$ times to be able to declare that it is $\epsilon$-optimal with probability at least $1-\delta$.

Similar to the work of~\cite{VA-PV-JW:87} extending Lai and Robbins' bounds \cite{TLL-HR:85} to the case of multiple plays, Kalyanakrishnan \etal~\cite{SK-AT-PA-PS:12} extended the work of \cite{EED-SM-YM:06} from finding the $\epsilon$-optimal arm to finding the $m$ $\epsilon$-\emph{best} arms with probability at least $1-\delta$. In~\cite{SK-AT-PA-PS:12} this problem is called Explore-$m$, and an algorithm that solves it $(\epsilon,m,\delta)$-optimal. Note that the problem in \cite{EED-SM-YM:06} is the special case Explore-1. The Explore-$m$ problem is studied in~\cite{SK-AT-PA-PS:12} for rewards that are Bernoulli distributed. It is proved that, for every $(\epsilon,m,\delta)$-optimal algorithm, there exists a bandit problem on which that algorithm has worst-case sample complexity of at least $\log(m/8\delta)$. Specifically, it is shown that there exists a bandit problem such that the number of samples $T^*$ required to identify $m$ $\epsilon$-best arms obeys
\beq \label{eq:exploremBound}
T^* \geq \frac{1}{18375} \frac{N}{\epsilon^2} \log \left( \frac{m}{8 \delta} \right).
\eeq
This gives a worst-case bound on the number of times all arms need to be sampled to achieve $(\epsilon, m, \delta)$-optimality.

The bounds \eqref{eq:explore1Bound} and \eqref{eq:exploremBound} were both formulated for the case of Bernoulli rewards, but it is straightforward to extend them to the case where the rewards are Gaussian distributed with unknown mean and known variance.

\subsection{Gaussian rewards}
In this paper we focus on the case of Gaussian reward distributions, where the distribution $\nu_i$ of rewards associated with arm $i$ is Gaussian with mean $m_i$, which is unknown to the decision maker, and variance $\sigma_{s,i}^2$, which is known to the decision maker from, e.g., previous observations or known measurement characteristics. Relaxation of the assumption of known variance is discussed  in Remark~\ref{remark-unknown-variance}.  For the given case, the Kullback-Leibler divergence in \eqref{eq:LaiRobbinsBound} takes the value
\beq \label{eq:LRBoundGaussian}
D(\nu_i || \nu_{i^*}) = \frac{1}{2} \left( \frac{\Delta_i^2}{\sigma_{s,i^*}^2} +  \frac{\sigma_{s,i}^2}{\sigma_{s,i^*}^2} - 1 - \log \frac{\sigma_{s,i}^2}{\sigma_{s,i^*}^2} \right).
\eeq

This equation is more easily interpreted when the reward variances are uniform, i.e., $\sigma_{s,i}^2 = \sigma_s^2$ for each $i$. In some cases we assume uniform variance for simplicity of exposition, but the relevant results are readily generalized to the case of non-uniform variance. Assuming uniform variance,  
$D(\nu_i || \nu_{i^*}) = {\Delta_i^2/2 \sigma_{s}^2}$,
so the bound \eqref{eq:LaiRobbinsBound} is
\beq \label{eq:LaiRobbinsGaussian}
\E{n_i^T} \geq \left( \frac{2 \sigma_s^2}{\Delta_i^2} + o(1) \right) \log T.
\eeq
This result can be interpreted as follows. For a given value of $\Delta_i$, a larger variance $\sigma_s^2$ makes the rewards more variable and therefore it is more difficult to distinguish between the arms. For a given value of $\sigma_s^2$, a larger value of $\Delta_i$ makes it easier to distinguish it from the optimal arm. The expressions for the problem with multiple plays (i.e., \eqref{eq:regretBoundMultiplePlays}) are identical except for substituting $\sigma(k)$ for $i^*$ and $\Delta^{(k)}_i$ for $\Delta_i$.

\section{The multi-armed bandit problem with satisficing objectives} \label{sec:satisficingObjectives}
We now define the multi-armed bandit problem with satisficing objectives. We propose several new satisficing notions of regret and find associated bounds on optimal performance. These notions capture two dimensions of the satisficing problem: \emph{satisfaction}, i.e., the agent's desire to obtain a reward that is above a certain threshold, and \emph{sufficiency}, i.e., the agent's desire to attain a level of confidence that its choice of a given arm will bring them satisfaction. We define these notions first for satisficing in mean reward and then extend them to satisficing in instantaneous reward, which we refer to as \emph{robust satisficing}. 
 
\subsection{Satisficing in mean reward}
\label{sec:thresholdMit}
We define \emph{satisfaction} in mean reward as having an expected reward $m_{i_t}$ that is above a specified threshold value $\satThresh$. Formally, we represent satisfaction in mean reward at time $t$ by the variable $s_t$, defined as
\beq \label{eq:satisfaction} s_t = 
\begin{cases} 1, & \text{if } m_{i_{t}} \geq \satThresh \\ 0, & \text{otherwise}. \end{cases} \eeq

The threshold $\satThresh$ is a free parameter that must be specified by the decision-making agent. Let $m_{i^*} = \max_i m_i$ be the maximum expected reward from any arm. The agent can never be satisfied if   $\satThresh$ is greater than $m_{i^*}$, so we assume that $\satThresh \leq m_{i^*}$ to make the problem feasible. If $\satThresh > m_{\sigma(2)}$, i.e., greater than the mean reward of the second-best arm, then arm $\sigma(1) = i^*$ is the only one that is satisfying in mean reward.

As in the multi-armed bandit problem with multiple plays, let $\sigma$ be a permutation of $\{1, \ldots, N\}$ such that $m_{\sigma(1)} \geq m_{\sigma(2)} \geq \cdots m_{\sigma(N)}$. Let $k$ be the largest integer such that $m_{\sigma(k)} \geq \satThresh$. The arms $\{ \sigma(1), \ldots, \sigma(k) \}$ are the $k$-best arms defined by the satisfaction threshold $\satThresh$. For each arm $i$, define the thresholded expected regret $\Delta_i^{\satThresh} = \max\{ \satThresh-m_i, 0 \}$.  For each $k$-best arm, the thresholded regret is zero, and for each $k$-worst arm $i \in \{ \sigma(k+1), \ldots, \sigma(N) \}$, the value $\Delta_i^{\satThresh} > 0$ quantifies the extent to which the arm is unsatisfying in mean rewards. Note that if $\satThresh = m_{i^*}$, $\Delta_i^{\satThresh} = \Delta_i$, which is the standard measure of expected regret. We refer to the $k$-best and $k$-worst arms as \emph{satisfying} and \emph{non-satisfying} arms, respectively.

The satisfaction variable $s_t$ defined in \eqref{eq:satisfaction} can be written as a function of the sign of $\Delta_{i_t}^{\satThresh}:$ 
\[ s_t = \begin{cases}
1, & \text{if } \Delta_{i_t}^{\satThresh} = 0,\\
0, & \text{otherwise.}
\end{cases}
\]
The quantity $s_t$ is deterministic. However, since the agent does not know the value of $\Delta_i^{\satThresh}$ associated with any given arm, they must learn it by sampling rewards from the various arms and updating their beliefs accordingly. 
Adopting a Bayesian framework, we assume $s_t$ is a realization of a binary random variable $S_t$. 
Due to the stochastic nature of the rewards the agent will have less than perfect confidence in their beliefs about the value of $s_t$.

We distinguish satisficing objectives in mean reward according to the degree $\delta \in [0,1]$ of confidence the agent \red{seeks} in their beliefs, which we call \emph{sufficiency in mean reward}. We define an arm $i$ to be \emph{($\delta$-)sufficing in mean reward} if
\[ \Pr{S_t = 1 } \geq 1-\delta, \]
where the probability is evaluated based on the agent's current beliefs. For non-zero values of $\delta$, the agent finds it sufficient to have finite confidence that they are satisfied, while for $\delta = 0$, the agent wants certainty that they are satisfied. The agent cannot achieve certainty in finite time, so these two cases result in qualitatively different behavior: $\delta = 0$ means the agent will never stop exploring, while $\delta > 0$ means the agent will settle on a set of acceptable options after finite time.

The \emph{satisficing-in-mean-reward objective} is 
\begin{equation}\label{eq:satisficing-in-mean-obj}
\sum_{t=1}^T \boldsymbol 1\big( (s_t =1) \textup{ or } \Pr{S_t =1} > 1 -\delta \big), 
\end{equation}
where $\boldsymbol 1(\cdot )$ is the indicator function. The objective~\eqref{eq:satisficing-in-mean-obj} is maximized if, at each time, a satisfying option is selected, or the probability that the option is satisfying is sufficiently high. The event that an option is satisfying is not known a priori and must be learned by exploration. This results in an explore-exploit tradeoff as in the standard multi-armed bandit problem.

To quantify the optimal explore-exploit tradeoff in the spirit of the Lai-Robbins bound we introduce the following notion of the \emph{expected satisficing regret} at time $t$, $R_t$, defined by
\beq \label{eq:satisficingRegret}
R_t = \begin{cases}
\Delta_{i_t}^{\satThresh}, & \text{if } \Pr{S_t = 1} \leq 1 - \delta,\\
0, & \text{otherwise}.
\end{cases}
\eeq
If the agent is insufficiently certain of being satisfied by the choice of $i_t$, they incur expected regret of $\Delta_{i_t}^{\satThresh}$. Otherwise, they incur no regret.

We define the satisficing-in-mean-reward multi-armed bandit problem in terms of minimizing cumulative expected satisficing regret.

\begin{defn}[Satisficing-in-mean-reward multi-armed bandit problem] \label{prob:satInMeanReward}
The \emph{satisficing-in-mean-reward multi-armed bandit problem} is to minimize the cumulative sum of the expected satisficing regret \eqref{eq:satisficingRegret}:
\beq \label{eq:cumSatisficingRegret}
J_R = \E{\sum_{t=1}^T R_t}.
\eeq
\end{defn}

The satisficing-in-mean-reward  bandit problem has two parameters: $\satThresh$ and $\delta$. These parameters characterize the agent's thresholds for satisfaction and sufficiency, respectively. For purposes of analysis we distinguish four cases as a function of the parameter values. For the satisfaction threshold $\satThresh \in \bbR$, the first case is setting $\satThresh > m_{\sigma(2)}$,  while the second case is setting $\satThresh \le m_{\sigma(2)}$. For the sufficiency threshold $\delta \in [0,1]$, the first case is the certainty value $\delta = 0$, while the second case is $\delta \in (0,1]$.

Table \ref{tab:thresholdm} summarizes the four problems that result from the interaction of the two dimensions of satisfaction and sufficiency. Problem 1 sets the satisfaction threshold $\satThresh > m_{\sigma(2)}$ and the sufficiency threshold $\delta = 0$, which results in a standard  bandit problem. We call Problem 2 with $\satThresh \le m_{\sigma(2)}$ and $\delta = 0$ \emph{satisfaction-in-mean-reward}. We call Problem 3 with $\satThresh > m_{\sigma(2)}$ and $\delta \in (0,1]$ \emph{$\delta$-sufficing}. Finally, we call Problem 4 with $\satThresh \le m_{\sigma(2)}$ and $\delta \in (0,1]$, \emph{$(\satThresh,\delta)$-satisficing}.

\begin{remark}We note that the distinction between Problems 1 and 2 and between Problems 3 and 4 is only due to the range of values $\satThresh$ can take. These problems can be thought of as a single problem in which the choice of $\satThresh$ dictates the cardinality of the set of satisfying arms. However, the two ranges of thresholds $\satThresh > m_{\sigma(2)}$ and $\satThresh \le m_{\sigma(2)}$ allow us to clearly contrast the satisficing problem with the standard problem. Assuming $\satThresh > m_{\sigma(2)}$ in Problems 1 and 3 is equivalent to assuming that the agent seeks the (unknown) highest mean reward, which is consistent with the standard problem.  The policies we define for Problems 1 and 3 do not rely on a known threshold $\satThresh$.     Assuming $\satThresh \leq m_{\sigma(2)}$ is equivalent to assuming that the agent seeks to meet a (known) desired mean reward threshold. The policies we define for Problems 2 and 4 do  rely on the  threshold $\satThresh$.    These same assumptions analogously distinguish Problems 5 and 7 from Problems 6 and 8 defined in Section~\ref{InstRewardProblems}.  However, unlike the policies for Problems 1 and 3, the policies defined for Problems 5 and 7 do rely on $\satThresh > m_{\sigma(2)}$ being known.
We do not assume in any of the problems that the agent  knows the permutation $\sigma$, so no policies depend on $\sigma$.
\end{remark}

We develop performance bounds for each of these problems in terms of corollaries of the performance bounds presented in Section II. For the problems with $\delta = 0$, these bounds show that cumulative expected regret must grow at least at a logarithmic rate, while for the problems with $\delta > 0$, finite regret is possible.

\textbf{Problem 1: Standard bandit} The satisficing-in-mean-reward multi-armed bandit problem with $\satThresh > m_{\sigma(2)}$ and $\delta = 0$ is a standard multi-armed bandit problem. Therefore, for this problem, the Lai-Robbins bound \eqref{eq:LaiRobbinsBound} holds, and the expected number of times a suboptimal arm $i$ is chosen obeys
\[ \E{n_i^T} \geq \left( \frac{1}{D(\nu_i || \nu_{i^*})} + o(1) \right) \log T. \]
As a direct consequence, the cumulative expected satisficing regret \eqref{eq:cumSatisficingRegret} grows at least logarithmically with time horizon $T$:
\[ J_R \geq \left( \sum_{i=1}^N \frac{\Delta_i^\satThresh}{D(\nu_i || \nu_{i^*})} + o(1) \right) \log T. \]

\textbf{Problem 2: Satisfaction-in-mean-reward} The satisfaction-in-mean-reward problem, defined as the satisficing-in-mean-reward multi-armed bandit problem where $\satThresh \le m_{\sigma(2)}$ and $\delta = 0$, also has a logarithmic lower bound on the cumulative expected satisficing regret:
\begin{corollary}[Satisfaction-in-mean-reward regret bound]
The satisfaction-in-mean-reward problem is a satisficing-in-mean-reward multi-armed bandit problem where the objective \eqref{eq:cumSatisficingRegret} is defined with $\satThresh \le m_{\sigma(2)}$ and $\delta = 0$. Any policy solving the satisfaction-in-mean-reward problem obeys
\beq \label{eq:satisficingInTheMeanBound}
\E{n_i^T} \geq \left( \frac{1}{D(\nu_i || \nu_{\sigma(k)})} + o(1) \right) \log T
\eeq
for each non-satisfying arm $i$, where $\sigma$ is a permutation of $\{1, \ldots, N \}$ such that $m_{\sigma(1)} \geq m_{\sigma(2)} \geq \cdots \geq m_{\sigma(N)}$ and $k$ is the largest integer such that $m_{\sigma(k)} \geq \satThresh$.
\end{corollary}
\begin{proof}
The definition of satisfaction \eqref{eq:satisfaction} implies that performance bounds for the satisfaction-in-mean-reward problem and the multi-armed bandit problem with multiple plays are equivalent. Given a problem instance, the threshold $\satThresh$ induces the number $k$ of satisfying arms, so performance can be analyzed as in the problem with multiple plays. The bound \eqref{eq:regretBoundMultiplePlays} applies to the problem with multiple plays and the equivalence implies the result.
\end{proof}

\textbf{Problem 3: $\boldsymbol\delta$-sufficing} The $\delta$-sufficing problem, defined as the satisficing-in-mean-reward multi-armed bandit problem where $\satThresh > m_{\sigma(2)}$ and $\delta \in (0,1]$, admits policies that achieve cumulative expected regret that is a bounded function of $T$:
\begin{corollary}[$\delta$-sufficing regret bound] \label{cor:deltaSufficingRegretBound}
The $\delta$-sufficing problem is a satisficing-in-mean-reward multi-armed bandit problem where the objective \eqref{eq:cumSatisficingRegret} is defined with $\satThresh > m_{\sigma(2)}$ and $\delta \in (0,1]$. Any policy solving the $\delta$-sufficing problem obeys
\beq \label{eq:deltaSufficingRegretBound}
n_i^T \geq \bigO{\frac{1}{\epsilon^2} \log(1/\delta)}
\eeq
for each suboptimal arm $i$, where $\epsilon = \Delta_i = \satThresh-m_i$.
\end{corollary}
\begin{proof}
The definition of satisfaction \eqref{eq:satisfaction} in the $\delta$-sufficing problem implies that the agent incurs regret if the arm selected is not $(\epsilon=0, \delta)$-optimal.  The bound \eqref{eq:explore1Bound} thus provides a lower bound on the number of times the agent must incur regret.
\end{proof}

\textbf{Problem 4: $\boldsymbol{(\satThresh,\delta)}$-satisficing} The $(\satThresh,\delta)$-satisficing problem, defined as the satisficing-in-mean-reward multi-armed bandit problem where $\satThresh \le m_{\sigma(2)}$ and $\delta \in (0,1]$, admits policies that achieve cumulative expected regret that is a bounded function of $T$:
\begin{corollary}[$(\satThresh,\delta)$-satisficing regret bound] \label{cor:MdeltaSufficingRegretBound}
The $(\satThresh,\delta)$-satisficing problem is a satisficing-in-mean-reward multi-armed bandit problem where the objective \eqref{eq:cumSatisficingRegret} is defined with $\satThresh \le  m_{\sigma(2)}$ and $\delta \in (0,1]$. Any policy solving the $(\satThresh,\delta)$-satisficing multi-armed bandit problem obeys
\beq \label{eq:MdeltaSufficingRegretBound}
T^* = \sum_{i=1}^N n_i^{T^*} \geq \frac{1}{18375} \frac{N}{\epsilon^2} \log \left( \frac{k}{8\delta} \right)
\eeq
where $\sigma$ is a permutation of $\{1, \ldots, N \}$ such that $m_{\sigma(1)} \geq m_{\sigma(2)} \geq \cdots \geq m_{\sigma(N)}$, $k$ is the largest integer such that $m_{\sigma(k)} \geq \satThresh$, and $\epsilon = \satThresh-m_{\sigma(k)}$. Since only arms in $\{\sigma(k+1), \ldots, \sigma(N)\}$ result in regret, the left hand side of \eqref{eq:MdeltaSufficingRegretBound} is an upper bound on the expected satisficing regret \eqref{eq:cumSatisficingRegret}.
\end{corollary}
\begin{proof}
The definition of satisfaction \eqref{eq:satisfaction} in the $(\satThresh,\delta)$-sufficing problem implies that an algorithm that minimizes satisficing regret is equivalent to an $(\epsilon = m_{\sigma(k)}-\satThresh, k, \delta)$-optimal algorithm in the sense of \cite{SK-AT-PA-PS:12}. Therefore, the bound \eqref{eq:exploremBound} applies to the $(\satThresh,\delta)$-sufficing problem.
\end{proof}

Recall that $T^*$ is the number of times all arms (including the optimal one) should by cumulatively sampled such that following $T^*$ an $(\satThresh, \epsilon)$-optimal decision can be made. The lower bounds on both $T^*$ and $n_i^{T^*}$ are independent of $T$, suggesting that for $(\satThresh, \epsilon)$-satisficing, a bounded regret can be achieved.

Corollaries \ref{cor:deltaSufficingRegretBound} and \ref{cor:MdeltaSufficingRegretBound} show that the worst-case regret is a bounded function of $T$ for the sufficing problems, where $\delta > 0$. Therefore we can conclude that the expected regret for such problems can also be a bounded function of $T$. This is an important distinction from the maximizing problems, where $\delta = 0$: in such problems, the Lai-Robbins bound \eqref{eq:LaiRobbinsBound} implies that the expected regret must grow logarithmically with $T$. As is standard in the bandit literature, we say an algorithm has efficient performance if its regret matches, up to constant factors, the relevant growth rates: $\log T$ for maximizing problems and $\log(k/\delta)/\epsilon^2$ for sufficing problems.

\subsection{Robust satisficing in instantaneous reward}
\label{InstRewardProblems}
The four objectives defined in Section \ref{sec:thresholdMit} above define satisfaction \eqref{eq:satisfaction} in terms of the mean reward $m_i$ from an arm $i$. This captures situations where the time scale for satisfaction spans numerous decision times. For example, consider foraging, where an animal must consume a minimum amount of food each day. If each decision time represents a small portion of the day, the total food consumed during the day represents the sum of numerous small rewards from each decision time. As long as the mean reward at each decision time is sufficiently high, the animal will meet its daily food requirement.

If, instead, the decision time scale is the same as the satisfaction time scale, it is more appropriate to define satisfaction at time $t$ in terms of the reward $r_t$ received at that time. This requires more robust algorithms, in the sense that they must ensure that each reward, rather than simply the mean reward, is satisfying with high probability. In this context we define satisfaction in two stages. First, we define \emph{happiness} as receiving a reward $r_t$ that is at least a threshold value $M \in \bbR$. We represent happiness at time $t$ as the Bernoulli random variable $h_t$, defined as
\beq \label{eq:happiness}
h_t = \begin{cases}
1, & \text{if } r_t \geq M\\
0, & \text{otherwise}.
\end{cases}
\eeq
We define the \emph{success probability} of the happiness random variable $h_t$ as
\beq \label{eq:happinessProbability}
p_i = \Pr{h_t = 1 | i_t = i}.
\eeq
The success probability $p_i$ is the expected rate of happiness due to picking arm $i$. This defines a Bernoulli multi-armed bandit problem where the mean reward (i.e., happiness rate) is $p_i$. We then define satisfaction in terms of a threshold $\Pi$ for this Bernoulli multi-armed bandit problem as we did in \eqref{eq:satisfaction}:
\beq \label{eq:satisfactionR} 
s_t = \begin{cases}
1, & \text{if } p_{i_t} \geq \Pi\\
0, & \text{otherwise}.
\end{cases} 
\eeq

Given the happiness threshold $M$, this definition is identical to the definition \eqref{eq:satisfaction} of satisfaction where $m_i = p_i$, $p_{i^*} = \max_i p_i$, and $\satThresh = \Pi$. Therefore the four satisficing multi-armed bandit problems defined in Table \ref{tab:thresholdm} can be used to define four additional problems in this context, which we call \emph{robust satisficing}. 

\begin{defn}[Robust satisficing multi-armed bandit problem]
The \emph{robust satisficing multi-armed bandit problem} is to minimize the cumulative sum of the expected satisficing regret \eqref{eq:satisficingRegret}:
\[ J_R = \E{\sum_{t=1}^T R_t}, \]
where the regret $R_t$ is defined using the notion of satisfaction defined by \eqref{eq:happiness}--\eqref{eq:satisfactionR}.
\end{defn}

A robust satisficing multi-armed bandit problem has three parameters: $M, \Pi,$ and $\delta$. We assume that $M$ and $\Pi$ are chosen such that there is at least one satisfying arm; otherwise, the expected regret must grow indefinitely. Table \ref{tab:thresholdr} summarizes the four robust satisficing multi-armed bandit problems that result from the interaction of the two dimensions of satisfaction and sufficiency, which we list below. We assume that $\varsigma$ is a permutation of $\until{N}$ such that $p_{\varsigma(1)} \ge p_{\varsigma(2)} \ge \ldots p_{\varsigma(N)}$. 

\textbf{Problem 5: Robust bandit} The robust bandit problem is defined as the robust satisficing multi-armed bandit problem where $\Pi > p_{\varsigma(2)}$ and $\delta = 0$.

\textbf{Problem 6: Robust satisfaction} The robust satisfaction problem is defined as the robust satisficing multi-armed bandit problem where $\Pi \le p_{\varsigma(2)}$ and $\delta = 0$.

\textbf{Problem 7: $\delta$-robust sufficing} The $\delta$-robust sufficing problem is defined as the robust satisficing multi-armed bandit problem where $\Pi > p_{\varsigma(2)}$ and $\delta \in (0,1]$.

\textbf{Problem 8: $(\Pi,\delta)$-robust satisficing} The $(\Pi,\delta)$-robust satisficing problem is defined as the robust satisficing multi-armed bandit problem where $\Pi \le p_{\varsigma(2)}$ and $\delta \in (0,1]$.

For a large class of reward distributions, there is an equivalence between Problems 5--8 defined in terms of $r_t$ and Problems 1--4 defined in terms of $m_i$. By Lemma \ref{lem:equivalence} below, when the rewards $r_t$ follow a Gaussian distribution with unknown mean $m_i$ and known variance $\sigma_{s,i}^2$, each problem in Table \ref{tab:thresholdr} is equivalent to the analogous problem in Table \ref{tab:thresholdm}.

\begin{table}
\centering
\begin{tabular}{c | c | c |}
 Threshold level & Seek certainty ($\delta = 0$) & Suffice ($\delta > 0$)\\
 \hline
 $\satThresh > m_{\sigma(2)}$ & 1) Standard bandit  & 3) $\delta$-sufficing\\
 \hline
 $\satThresh \le m_{\sigma(2)}$ & 2) Satisfaction-in-mean-rwd & 4) $(\satThresh,\delta)$-satisficing\\
 \hline
\end{tabular}
\vspace{0.1in}
\caption{Table of the four different regret concepts, and resulting problems, associated with the satisficing-in-mean-reward multi-armed bandit problem.\label{tab:thresholdm}}
\end{table}

\begin{table}
\centering
\begin{tabular}{c | c | c |}
Threshold level & Seek certainty ($\delta = 0$) & Suffice ($\delta > 0$)\\
 \hline
 $\Pi > p_{\varsigma(2)}$ & 5) Robust bandit  & 7) $\delta$-robust sufficing\\
 \hline
 $\Pi \le  p_{\varsigma(2)}$  & 6) Robust satisfaction & $8) (\Pi,\delta)$-robust satisficing\\
 \hline
\end{tabular}
\vspace{0.1in}
\caption{Table of the four different regret concepts, and resulting problems, associated with the robust satisficing multi-armed bandit problem. The quantity $p_i$ represents the probability of happiness (i.e., receiving a reward of at least $M$) due to choosing arm $i$.\label{tab:thresholdr}}
\vspace{-9mm}
\end{table}

\section{Satisficing with Gaussian rewards}\label{sec:satisficingGaussian}
In this section we study the Gaussian satisficing multi-armed bandit problem. This is the satisficing multi-armed bandit problem where the reward $r_t$ due to selecting arm $i_t$ is $r_t \sim \mcN(m_{i_t},\sigma_{s,i_t}^2)$ and $\sigma_{s,i_t}^2$ is the known variance of arm $i_t$. In this case, we show a formal equivalence between the satisficing-in-mean-reward multi-armed bandit problems and the robust satisficing multi-armed bandit problems. The choice of Gaussian rewards facilitates modeling correlation dependencies among arms, which can be useful in applications.

\subsection{Equivalence lemma for Gaussian rewards}
For the Gaussian robust satisficing multi-armed bandit problem, define the quantity
\beq \label{eq:defX}
x_i = \frac{m_i - M}{\sigma_{s,i}},
\eeq
which we call the \emph{standardized mean reward}, for each arm $i$. The following lemma states that each Gaussian robust satisficing multi-armed bandit problem where satisfaction is defined by \eqref{eq:satisfactionR} is equivalent to a Gaussian satisficing-in-mean-reward multi-armed bandit problem where satisfaction is defined by \eqref{eq:satisfaction} with standardized reward distributions.

\begin{lemma}[Equivalence for Gaussian rewards]\label{lem:equivalence}
Each Gaussian robust satisficing multi-armed bandit problem is equivalent to a Gaussian satisficing-in-mean-reward multi-armed bandit problem with rewards $\tilde{r}_t \sim \mcN(x_{i_t},1)$ with $x_i$ given by \eqref{eq:defX}. That is, the ordering of the arms in terms of $x_i$ is identical to the ordering in terms of $p_i$, and, in particular, the arm with maximal $x_i$ is the arm with maximal $p_i$.
\end{lemma}
\begin{proof}
With Gaussian rewards, the probability \eqref{eq:happinessProbability} of happiness due to choosing arm $i$ is
\begin{align}
p_i &= \Pr{m_{i} + \sigma_{s,i} z \geq M} \nonumber \\
& = \Phi\left(\frac{m_{i}-M}{\sigma_{s,i}}\right) = \Phi(x_i), \label{eq:xpi}
\end{align}
where $z \sim \mcN(0,1)$ is a standard normal random variable and $\Phi(z)$ is its cumulative distribution function. Let $i^* = \argmax_i p_i$. The key insight is that $\Phi(\cdot)$ is a monotonically increasing function, which implies that the ordering of arms in terms of $p_i$ is identical to the ordering in terms of $x_i$. In particular, arm $i^*$ is the arm with maximal $x_i$. Therefore, satisfaction in terms of $r_t$ is equivalent to satisfaction in terms of the mean reward $x_i$.

This is again a Gaussian bandit problem: consider the standardized reward
\beq \label{eq:standardize} \tilde{r}_t = \frac{r_t - M}{\sigma_{s,i_t}}, \eeq
which is a Gaussian random variable $\tilde{r}_t \sim \mcN(x_{i_t},1).$ The quantity $x_i$ plays the role of the mean reward $m_i$ and the transformed rewards have uniform variance $\tilde{\sigma}_s^2 = 1$. Minimizing the robust satisficing regret in terms of $r_t$ is equivalent to minimizing the satisficing regret in terms of $x_{i}$.
\end{proof}

Lemma \ref{lem:equivalence} has two implications for the relationship between Problems 5--8 and Problems 1--4 when rewards are Gaussian distributed. First, each Problem 5--8 inherits a regret bound from the corresponding Problem 1--4. Second, each Problem 5--8 can be solved by applying the algorithm developed for Problem 1--4 by first applying the standardization transformation \eqref{eq:standardize} to the observed rewards.

\begin{remark}[Location-scale families]
Lemma \ref{lem:equivalence} is easily generalized to reward distributions belonging to location-scale families. A location-scale family is a set of probability distributions closed under affine transformations, i.e., if the random variable $X$ is in the family, so is the variable $Y = a + bX,$ where $a,b \in \bbR$. Any random variable $X$ in such a family with mean $\mu$ and standard deviation $\sigma$ can be written as $X = \mu + \sigma Z$, where $Z$ is a zero-mean, unit-variance member of the family. Examples include the uniform and Student's $t$-distributions.
\end{remark}

\subsection{Application to the Gaussian robust satisficing problems}
In this section we show how to use the equivalence result of Lemma \ref{lem:equivalence} for the full set of robust satisficing problems in the case of Gaussian rewards.

Recall from Lemma \ref{lem:equivalence} that the probability of happiness \eqref{eq:happinessProbability} due to picking an arm $i$ is $p_i$. In the proof of the lemma, we show that maximizing the probability of happiness is equivalent to maximizing the mean reward in a Gaussian multi-armed bandit problem with mean rewards $x_i = \Phi^{-1}(p_i)$, where $x_i$ is the standardized mean reward $(m_i-M)/\sigma_{s,i}$. Given an algorithm developed for one of the Problems 1--4 defined in Table \ref{tab:thresholdm}, it can be applied to the corresponding Problem 5--8 defined in Table \ref{tab:thresholdr} as follows. Standardize the observed rewards $r_t$ and run the algorithm using the standardized rewards $\tilde{r}_t = (r_t-M)/\sigma_{s,i_t}$ as input. For example, Problem 5, the robust multi-armed bandit problem, can be solved by an algorithm designed to solve Problem 1, the standard bandit problem, where rewards are transformed according to \eqref{eq:standardize} before being input to the algorithm. The same procedure allows one to apply algorithms developed for Problem 3, $\delta$-sufficing, to Problem 7, $\delta$-robust sufficing. 

For Problem 6, robust satisfaction, and Problem 8, $(\Pi,\delta)$-robust satisficing, we need a threshold $X$ that is analogous to the threshold $\satThresh$ defined for Problem 2, satisfaction-in-mean-reward, and Problem 4, $(\satThresh,\delta)$-satisficing. We use the relationship between $x_i$ and $p_i$ to derive the threshold. In particular, for a robust satisficing problem with probability of happiness threshold $\Pi$, define the threshold $X$ by
\beq \label{eq:thresholdX}
X = \Phi^{-1}(\Pi).
\eeq
When the rewards are Gaussian distributed, we can apply algorithms developed for Problems 2 and 4 to the corresponding robust satisficing Problems 6 and 8 by standardizing rewards and using the threshold $X$ defined in \eqref{eq:thresholdX} in place of the threshold $\satThresh$.

Lemma \ref{lem:equivalence} implies that the efficient performance guarantees for algorithms designed for Problems 1--4 also hold when they are used to solve the robust satisficing Problems 5--8.

\section{The UCL algorithm for Gaussian multi-armed bandit problems}
In this section we review the UCL algorithm, a Bayesian algorithm we developed and analyzed in~\cite{PR-VS-NEL:14}
to solve the standard Gaussian multi-armed bandit problem. The UCL algorithm was developed by applying the Bayesian upper confidence bound approach of \cite{EK-OC-AG:12} to the case of Gaussian rewards; the choice of Gaussian rewards facilitated the modeling of human decision-making behavior.

The UCL algorithm maintains a belief about the mean rewards $\bfm$ by starting with a prior and updating it using Bayesian inference as new rewards are received. At each time $t$ the algorithm chooses arm $i_t$ using a heuristic that is a simple function of the current belief state. For uninformative priors, the UCL algorithm achieves logarithmic regret, i.e., optimal performance.

Uninformative priors correspond to having no information about the mean rewards. A major advantage of the UCL algorithm is its ability to incorporate information about the mean rewards through the use of a so-called informative prior. In \cite{PR-VS-NEL:14}, we showed that an appropriately-chosen prior can significantly increase the performance of the UCL algorithm. Several different UCL algorithms were developed in \cite{PR-VS-NEL:14}, including a stochastic decision rule to model human behavior; here we cover only the deterministic UCL algorithm, which, for brevity, we refer to as the UCL algorithm.

\subsection{Prior}
The prior distribution captures the agent's knowledge about the vector of mean rewards $\bfm$ before beginning the task. We assume that the prior distribution is multivariate Gaussian with mean $\bfmu_0 \in \bbR^N$ and covariance $\Sigma_0 \in \bbR^{N \times N}$:
\beq \label{eq:prior}
\bfm \sim \mcN(\bfmu_0,\Sigma_0).
\eeq
The $i^{th}$ element of $\bfmu_0$, denoted by $\mu_i^0$, represents the agent's mean belief of the reward $m_i$ associated with arm $i$. The $(i,i)$ element of $\Sigma_0$, denoted by $\left(\sigma_{i}^0 \right)^2$, represents the agent's uncertainty associated with that belief. Off-diagonal elements of $\Sigma_0$, e.g., $\sigma_{ij}^0$, represent the agent's perceived relationship between $m_i$ and $m_j$: if $\sigma_{ij}^0$ is positive, high values of $m_i$ are correlated with high values of $m_j$, while if it is negative, high values of $m_i$ correlate with low values of $m_j$. Any positive-definite matrix can be used as $\Sigma_0$, but it is often useful to consider a structured parametrization, such as $\Sigma_0 = \sigma_0^2 \Sigma$, where $\sigma_0^2 > 0$ encodes the agent's uncertainty.
One important special case is an uncorrelated prior, where $\Sigma$ is diagonal, which corresponds to the agent perceiving the rewards associated with different arms to be independent. Another important special case is an uninformative prior, which corresponds to complete uncertainty, i.e., the limit $\sigma_0^2 \to +\infty$; an uninformative prior can be thought of as a special case of an uncorrelated prior.

\subsection{Inference update}
At each time $t$ the agent picks an arm $i_t$ and receives a reward $r_t$ that is Gaussian distributed: $r_t \sim \mcN(m_{i_t},\sigma_{s,i_t}^2)$. Bayesian inference provides an optimal solution to the problem of updating the belief state $(\bfmu_t,\Sigma_t)$ (i.e., the sufficient statistics for estimating $\bfm$) to incorporate this new information. Let $\Lambda_t = \Sigma_t^{-1}$, and let $\boldsymbol \phi_t \in \bbR^N$ be the vector with element $i_t$ equal to 1 and all other elements equal to zero. Then given the Gaussian prior \eqref{eq:prior}, the Bayesian update equations are linear \cite{SMK:93}:
\begin{align}\label{eq:inference}
\begin{split}
\mathbf{q} &= \frac{r_t \boldsymbol \phi_t}{\sigma_{s,i_t}^2} + \Lambda_{t-1} \boldsymbol \mu_{t-1} \\
\Lambda_{t} &= \frac{\boldsymbol \phi_t \boldsymbol \phi_t^T}{\sigma_{s,i_t}^2} + \Lambda_{t-1},  \\
\boldsymbol \mu_t &= \Sigma_t \mathbf{q}.
\end{split}
\end{align}

\subsection{Decision heuristic}
At each time $t$ the UCL algorithm computes a value $Q_i^t$ for each arm $i$. The algorithm then picks the arm $i_t$ that maximizes $Q_i^t$. That is, it picks
\beq \label{eq:it}
i_t = \arg \max_i Q_i^t.
\eeq
The heuristic value $Q_i^t$ is
\beq \label{eq:Q}
Q_i^t = \mu_i^t + \sigma_i^t \Phi^{-1}(1-\alpha_t),
\eeq
where $\mu_i^t = \left( \bfmu_t \right)_i$, $\left( \sigma_i^t\right)^2 = \left(\Sigma_t \right)_{ii}, \alpha_t = 1/(Kt),$ $K > 0$ is a tunable parameter, and $\Phi^{-1}(\cdot)$ is the quantile function of the standard normal random variable.
The heuristic $Q_i^t$ is a Bayesian upper limit for the value of $m_i$ based on the information available at time $t$. It represents an optimistic assessment of the value of $m_i$. The decision made can be thought of as the most optimistic one consistent with the current information.

\subsection{Performance}
In \cite{PR-VS-NEL:14}, we studied the case of homogeneous sampling noise (i.e., $\sigma_{s,i}^2 = \sigma_s^2$ for each $i$) and showed that the UCL algorithm achieves logarithmic cumulative expected regret uniformly in time. In particular, we proved that the following theorem holds. We define $\seqdef{\supscr{R}{UCL}_t}{t\in\until{T}}$ as the sequence of expected regret for the deterministic UCL algorithm.
\begin{theorem}[Regret of the deterministic UCL algorithm \cite{PR-VS-NEL:14}]\label{thm:UCL}
The following statements hold for the Gaussian multi-armed bandit problem and the deterministic UCL algorithm with uncorrelated uninformative prior and $K=1$:
\begin{enumerate}
\item the expected number of times a suboptimal arm $i$ is chosen until time $T$ satisfies
\begin{align*}
\E{n_{i}^T} \leq \Big( \frac{8 \sigma_s^2}{\Delta_i^2} + 2 \Big) \log T + 3\; ; 
\end{align*}
\item the cumulative expected regret until time $T$ satisfies
\begin{align*}
J_R = \sum_{t=1}^T \supscr{R_t}{}  \leq 
\sum_{i=1}^N \Delta_i \Bigg(\Big( \frac{8 \sigma_s^2}{\Delta_i^2} + 2 \Big) \log T + 3 \Bigg).
\end{align*}
\end{enumerate}
\end{theorem}

The implication of this theorem can be seen by comparing  1) with the Lai-Robbins bound \eqref{eq:LaiRobbinsGaussian}:  the UCL algorithm achieves logarithmic regret uniformly in time with a constant that differs from the optimal asymptotic one by a constant factor, and thus is considered to have optimal performance.

\section{Algorithms for satisficing Gaussian multi-armed bandit problems}
In this section we develop algorithms for solving Gaussian multi-armed bandit problems with the satisficing objectives proposed in Section \ref{sec:satisficingObjectives}. All the algorithms consist of modified versions of the UCL algorithm. We analyze the algorithms and show that they achieve efficient performance. The UCL algorithm solves the standard Gaussian multi-armed bandit problem, i.e., the satisficing Gaussian multi-armed bandit problem with $\satThresh > m_{\sigma(2)}$ and $\delta = 0$ (Problem 1). We develop three new UCL variants for Problems 2--4 in Table \ref{tab:thresholdm}. These algorithms can then be applied to Problems 5--8 in Table \ref{tab:thresholdr}. At the end of the section, we consider extensions to reward distributions other than the Gaussian with known variance.

\subsection{Problem 2: Satisfaction-in-mean-reward UCL algorithm}
\label{prob2alg}
A simple modification of the UCL algorithm achieves logarithmic regret for the Gaussian satisfaction-in-mean-reward problem, which is the satisficing-in-mean-reward multi-armed bandit problem with $\satThresh \le m_{\sigma(2)}$ and $\delta=0$ (Problem 2). We define this algorithm, which we refer to as the \emph{satisfaction-in-mean-reward UCL algorithm}, as follows.

As in \eqref{eq:Q}, define the heuristic value $Q_i^t$ as
\[ Q_i^t = \mu_i^t + \sigma_i^t \Phi^{-1}(1- \alpha_t), \]
where $\alpha_t = 1/(Kt)$ and $K>0$ is again a tunable parameter.

Let $\satThresh \in \bbR$ be the satisfaction threshold, so the agent is satisfied if it picks an arm with $m_i \geq \satThresh$. Let the eligible set at time $t$ be $\{ i \; |\; Q_i^t \geq \satThresh\}$.
In contrast to the UCL selection scheme \eqref{eq:it} that picks the arm with maximal $Q_i^t$, satisfaction-in-mean-reward UCL picks any arm in the eligible set.
That is, if the eligible set is non-empty, then 
\beq \label{eq:satisficingit}
i_t \in \{ i | Q_i^t \geq \satThresh \}, 
\eeq
or if the eligible set is empty, then satisfaction-in-mean-reward UCL picks the arm with maximal $Q_i^t$. 
Thus, if the most recently selected arm is in the eligible set, it may be selected again even if it does not have the maximal $Q_i^t$. 

The satisfaction-in-mean-reward UCL algorithm achieves logarithmic cumulative expected satisfaction-in-mean-reward regret, as guaranteed by the following theorem.
\begin{theorem}[Regret of the satisfaction-in-mean-reward UCL algorithm] \label{thm:satisficingInTheMean}
Let a Gaussian multi-armed bandit problem with the satisfaction-in-mean-reward objective have at least one arm $i$ that obeys $m_i > \satThresh$, and, without loss of generality, assume $\sigma_{s,i}^2 = 1$ for each arm $i$. Then, the following statements hold for the satisfaction-in-mean-reward UCL algorithm with uncorrelated uninformative prior and $K = 1$:
\begin{enumerate}
\item the expected number of times a non-satisfying arm $i$ is chosen until time $T$ satisfies
\begin{align*}
\E{n_i^T} &\leq \left( \frac{8}{\left( \Delta_i^{\satThresh} \right)^2} + 3 \right) \log T + 4;
\end{align*}
\item the cumulative expected satisfaction-in-mean-reward regret until time $T$ satisfies
\begin{align*}
J_{SM} &\leq \sum_{i=1}^N \Delta_i^{\satThresh} \left( \left( \frac{8}{\left( \Delta_i^{\satThresh} \right)^2} + 3 \right) \log T + 4 \right).
\end{align*}
\end{enumerate}
\end{theorem}

To prove Theorem \ref{thm:satisficingInTheMean} we  use the following bound from \cite{MA-IAS:64}.
\begin{lemma}[\bit{Bounds on the inverse Gaussian cdf}]\label{thm:tailBounds}
For the standard normal (i.e., Gaussian) random variable $z$ and a constant $w \in \bbR_{\geq 0},$
\beq \label{eq:subGaussianBound}
\Pr{z \geq w} \leq \frac{2 e^{-w^2/2}}{\sqrt{2 \pi}(w + \sqrt{w^2 + 8/\pi})} \leq \frac{1}{2} e^{-w^2/2}.
\eeq
It follows from \eqref{eq:subGaussianBound} that for any $\alpha \in [0.5, 1],$
\beq \label{eq:QDown}
\Phi^{-1}(1-\alpha) \leq \sqrt{-2 \log \alpha}.
\eeq

\end{lemma}

\begin{proof}[Proof of Theorem \ref{thm:satisficingInTheMean}]
The proof proceeds as in the proof of Theorem \ref{thm:UCL} in \cite{PR-VS-NEL:14}, which itself follows the proofs in \cite{PA-NCB-PF:02}. Let $i$ be a non-satisfying arm, i.e.,  $m_i < \satThresh$, and recall that $i^*$ designates the maximum mean reward. Then 
\begin{align*}
\E{n_i^T} &= \sum_{t=1}^T \Pr{i_t = i} \\
	&\leq \sum_{t=1}^T \Pr{Q_i^t \! \geq \! \satThresh} \! + \! \Pr{Q_i^t \geq Q_{i^*}^t
	\; \! \& \; \! \max_{j} Q_j^t \! < \! \satThresh}\\
	&\leq \eta + \sum_{t=1}^T \left(\Pr{Q_i^t \geq \satThresh, n_i^t \geq \eta}\right. \\
	&\qquad \qquad \quad \left. + \Pr{Q_i^t \geq Q_{i^*}^t, n_i^t \geq \eta}\right).
\end{align*}
The first term in the summand corresponds to the probability that the non-satisfying arm $i$ is in the eligible set, while the second term 
corresponds to the probability that the eligible set is empty and that a non-satisfying arm appears better than an optimal arm.

The statement $Q_i^t \geq Q_{i^*}^t$ implies that at least one of the following inequalities holds:
\begin{align}
\mu_i^t &\geq m_i + C_i^t \label{eq:muiHigh} \\
\mu_{i^*}^t &\leq m_{i^*} - C_{i^*}^t \label{eq:muiStarLow}\\
m_{i^*} &< m_i + 2 C_i^t, \label{eq:deltaSmall}
\end{align}
where $C_i^t = \sigma_i^t \Phi^{-1}(1-\alpha_t)$ and $\alpha_t = 1/(Kt)$. Otherwise, if none of \eqref{eq:muiHigh}--\eqref{eq:deltaSmall} holds, then
\[ Q_{i^*} = \mu_{i^*}^t + C_{i^*}^t > m_{i^*} \geq m_i + 2 C_i^t > \mu_i^t + C_i^t = Q_i^t. \]

We first analyze the probability that \eqref{eq:muiHigh} holds.
For an uncorrelated uninformative prior, $\mu_i^t$ is equal to $\bar{m}_i^t$, the empirical mean reward observed at arm $i$ until time $t$, and $\sigma_i^t = 1/\sqrt{n_i^t}$. Therefore, for an uncorrelated uninformative prior, 
\[ Q_i^t = \bar{m}_i^t + \frac{1}{\sqrt{n_i^t}} \Phi^{-1}(1-\alpha_t). \]

Conditional on $n_i^t$, the empirical mean reward $\bar{m}_i^t$ is itself a Gaussian random variable with mean $m_i$ and standard deviation $1/\sqrt{n_i^t}$, so  \eqref{eq:muiHigh} holds if
\begin{align*}
\bar{m}_i^t &\geq m_i + \frac{1}{\sqrt{n_i^t}} \Phi^{-1}(1-\alpha_t) \\
\Leftrightarrow m_i + \frac{z}{\sqrt{n_i^t}} &\geq m_i + \frac{1}{\sqrt{n_i^t}} \Phi^{-1}(1-\alpha_t)\\
\Leftrightarrow z & \geq \Phi^{-1}(1-\alpha_t),
\end{align*}
where $z \sim \mcN(0,1)$ is a standard normal random variable. Thus, for an uninformative prior,
$\Pr{ \eqref{eq:muiHigh} \text{ holds}} = \alpha_t = \frac{1}{Kt}$.

Similarly,  \eqref{eq:muiStarLow} holds if
\begin{align*}
\bar{m}_{i^*}^t &\leq m_{i^*} - C_{i^*}^t\\
\Leftrightarrow m_{i^*} + \frac{z}{\sqrt{n_i^t}} &\leq m_i - \frac{1}{\sqrt{n_i^t}} \Phi^{-1}(1-\alpha_t)\\
\Leftrightarrow z &\leq - \Phi^{-1}(1-\alpha_t),
\end{align*}
where $z \sim \mcN(0,1)$ is a standard normal random variable. Thus, for an uninformative prior,
$\Pr{ \eqref{eq:muiStarLow} \text{ holds}} = \alpha_t = \frac{1}{Kt}$.

Inequality \eqref{eq:deltaSmall} holds if 
\begin{align}
m_{i^*} &< m_i + \frac{2}{\sqrt{n_i^t}} \Phi^{-1}(1-\alpha_t)\nonumber \\
\Leftrightarrow \Delta_i &< \frac{2}{\sqrt{n_i^t}} \Phi^{-1}(1-\alpha_t) \;\;
\Leftrightarrow \frac{\Delta_i^2 n_i^t}{4} < - 2 \log \alpha_t\label{eq:quantileBoundApply}\\
\Rightarrow \frac{\Delta_i^2 n_i^t}{4} &< 2 \log t 
\;\; \Rightarrow \frac{\Delta_i^2 n_i^t}{4} < 2 \log T \nonumber
\end{align}
where $\Delta_i = m_{i^*} - m_i$ and inequality \eqref{eq:quantileBoundApply} follows from bound \eqref{eq:QDown}. Thus, for an uninformative prior,  \eqref{eq:deltaSmall} never holds if
\beq \label{eq:eta}
n_i^t \geq \frac{8}{\Delta_i^2} \log T.
\eeq
Thus, for $n_i^t$ sufficiently large, $\Pr{Q_i^t \geq Q_{i^*}^t} = 2/(Kt)$.

We now bound the probability $\Pr{Q_i^t \geq \satThresh}$ that a non-satisfying arm $i$ is in the eligible set. Note that $Q_i^t \geq \satThresh$ implies that at least one of the following inequalities holds:
\begin{align}
\mu_i^t &\geq m_i + C_i^t \label{eq:muiHighThresh}\\
\satThresh &< m_i + 2 C_i^t. \label{eq:deltaSmallThresh}
\end{align}
Otherwise, if neither \eqref{eq:muiHighThresh} nor \eqref{eq:deltaSmallThresh} holds,
$\satThresh \geq m_i + 2 C_i^t > \mu_i^t + C_i^t = Q_i^t $
and arm $i$ is not in the eligible set.

 \eqref{eq:muiHighThresh} is identical to \eqref{eq:muiHigh} and \eqref{eq:deltaSmallThresh} to  \eqref{eq:deltaSmall}. For an uninformative prior,
$ \Pr{ \eqref{eq:muiHighThresh} \text{ holds}} = \alpha_t = \frac{1}{Kt}$.  And
 \eqref{eq:deltaSmallThresh}  holds if
\begin{align*}
\satThresh &< m_i + \frac{2}{\sqrt{n_i^t}} \Phi^{-1}(1-\alpha_t) \\
\Leftrightarrow \Delta_i^{\satThresh} &< \frac{2}{\sqrt{n_i^t}} \Phi^{-1}(1-\alpha_t) \\
\Leftrightarrow \frac{\left(\Delta_i^{\satThresh}\right)^2 n_i^t}{4} &< -2\log(\alpha_t) \\
\Rightarrow \frac{\left(\Delta_i^{\satThresh}\right)^2 n_i^t}{4} &< 2 \log t \;\;
\Rightarrow \frac{\left(\Delta_i^{\satThresh}\right)^2 n_i^t}{4} < 2 \log T. \\
\end{align*}
Thus, for an uninformative prior,  \eqref{eq:deltaSmallThresh} never holds if
\beq \label{eq:etaThresh}
n_i^t \geq \frac{8}{\left( \Delta_i^{\satThresh} \right)^2} \log T .
\eeq

Since $m_{i^*} \geq \satThresh$, for each non-satisfying arm $i$, $\Delta_i^{\satThresh} \leq \Delta_i$. Thus, $1/\left( \Delta_i^{\satThresh} \right)^2 \geq 1/\Delta_i^2$ and  \eqref{eq:etaThresh} implies  \eqref{eq:eta}. So setting
\beq \label{eq:etaDef}
\eta = \left\lceil \frac{8}{\left( \Delta_i^{\satThresh} \right)^2} \log T \right\rceil
\eeq
yields the bound
\begin{align*}
\E{n_i^T} &\leq \eta + \sum_{t=1}^T \left( \Pr{Q_i^t \geq \satThresh, n_i^t \geq \eta} \right. \\
		& \qquad \qquad \quad + \left. \Pr{Q_i^t \geq Q_{i^*}^t, n_{i}^t \geq \eta} \right) \\
		& < \left\lceil \frac{8}{\left( \Delta_i^{\satThresh} \right)^2} \log T \right\rceil + 3 \sum_{t=1}^T \frac{1}{t}.
\end{align*}
The sum can be bounded by the integral
\[ \sum_{t=1}^T \frac{1}{t} \leq 1 + \int_1^T \frac{1}{t} \mathrm{d}t = 1 + \log T, \]
yielding the bound in the first statement of the theorem:
\begin{align*}
\E{n_i^T} &\leq \left( \frac{8}{\left( \Delta_i^{\satThresh} \right)^2} + 3 \right) \log T + 4.
\end{align*}
The second statement of the theorem follows from the definition \eqref{eq:satisficingRegret} of expected satisficing regret.
\end{proof}

\subsection{Problem 3: $\delta$-sufficing UCL algorithm}
\label{prob3alg}
An alternative modification of the UCL algorithm achieves finite satisficing regret in the Gaussian $\delta$-sufficing problem, which is the satisficing-in-mean-reward multi-armed bandit with $\satThresh > m_{\sigma(2)}$ and $\delta \in (0,1]$ (Problem 3). For the agent, this can be thought of as wanting to have finite confidence that it has found the unknown optimal arm ${\sigma(1)}$. For the $\delta$-sufficing problem, define the heuristic function
\[ Q_i^t = \mu_i^t + \sigma_i^t \Phi^{-1}\left( 1 - \frac{\delta}{2} \right). \]
We define the \emph{$\delta$-sufficing UCL algorithm} as the algorithm that selects arm $i_t = \argmax_i Q_i^t$ at each decision time $t$. The $\delta$-sufficing UCL algorithm achieves finite cumulative satisficing regret, as guaranteed by the following theorem.

\begin{theorem} \label{thm:deltaSufficing}
Consider the $\delta$-sufficing UCL algorithm with an uninformative prior. The number of times the picked arm $i_t$ is non-satisfying with probability greater than $\delta$ is upper bounded as
\[ n_i^T < \frac{4 \sigma_s^2}{\Delta_i^2} \left( \Phi^{-1}\left( 1 - \frac{\delta}{2} \right) \right)^2 + 1. \]
\end{theorem}
\begin{proof}
We bound $n_i^T$ by noting that a non-satisfying arm $i$ is picked only if $Q_i^t \geq Q_{i^*}^t$, which can be decomposed as in the proof of Theorem \ref{thm:satisficingInTheMean} into the three conditions
\begin{align}
\mu_i^t & \geq m_i + C_i^t \label{eq:muitHighDelta} \\
\mu_{i^*}^t & \leq m_{i^*} - C_{i^*}^t \label{eq:muiStartLowDelta} \\
m_{i^*} & < m_i + 2 C_i^t. \label{eq:gapDelta}
\end{align}

\eqref{eq:gapDelta} is equivalent to 
\[ \Delta_i = m_{i^*} - m_i < 2 C_i^t = \frac{2 \sigma_s}{\sqrt{n_i^t}} \Phi^{-1}(1-\delta/2). \]
Squaring  and rearranging, we see that this never holds if 
\[ n_i^t > \frac{4 \sigma_s^2}{\Delta_i^2} \left( \frac{\log(1/\delta)}{\log 2} + 1 \right) > \frac{4 \sigma_s^2}{\Delta_i^2} \left( \Phi^{-1}(1-\delta/2) \right)^2 = \eta. \]

The same argument as in the proof of Theorem \ref{thm:satisficingInTheMean} shows that for $n_i^t \geq 1$,  \eqref{eq:muitHighDelta} and \eqref{eq:muiStartLowDelta} each hold with probability at most $\delta/2$. Therefore, for $n_i^t > \eta + 1$, a non-satisfying arm is selected with probability at most $\delta$.
\end{proof}

Theorem \ref{thm:deltaSufficing} guarantees that the $\delta$-sufficing UCL algorithm achieves finite regret. Furthermore, the algorithm is efficient in that the regret matches the dependence on $\epsilon$ and $\delta$ in the bound \eqref{eq:deltaSufficingRegretBound}. To see this, note that a non-satisfying arm $i$ with $\Delta_i$ is an $\epsilon = \Delta_i$-suboptimal arm, so Corollary \ref{cor:deltaSufficingRegretBound} implies that $n_i^T$ is lower bounded by $\bigO{\log(1/\delta)/\epsilon^2}$. The statement of Theorem \ref{thm:deltaSufficing} combined with the bound \eqref{eq:QDown} on the inverse Gaussian cdf implies that $n_i^T$ is upper bounded by $8 \sigma_s^2 \log(2/\delta)/\Delta_i^2 + 1 = 8\sigma_s^2 \log(2/\delta)/\epsilon^2 + 1$, which matches the lower bound \eqref{eq:deltaSufficingRegretBound} up to constant factors.

\subsection{Problem 4: $(\satThresh,\delta)$-satisficing UCL algorithm}
\label{prob4alg}
A third modification of the UCL algorithm achieves finite satisficing regret in the Gaussian $(\satThresh,\delta)$-satisficing problem, which is the satisficing-in-mean-reward multi-armed bandit with $\satThresh \le m_{\sigma(2)}$ and $\delta \in (0,1]$ (Problem 4). For the agent, this can be thought of as wanting to have finite confidence that it has found an arm whose mean reward is above a known threshold. For the $(\satThresh,\delta)$-satisficing problem, define the heuristic function
\[ Q_i^t = \mu_i^t + \sigma_i^t \Phi^{-1} \left( 1- \frac{\delta}{3} \right). \]
Let the eligible set at time $t$ be $\{ i \;|\; Q_i^t \geq \satThresh \}$. 
We define the \emph{$(\satThresh,\delta)$-satisficing UCL algorithm} as the algorithm that selects arm $i_t \in \{ i | Q_i^t \geq \satThresh \}$, if the eligible set at time $t$ is non-empty. Otherwise, if the eligible set is empty, the algorithm picks the arm with maximal $Q_i^t$.

The $(\satThresh,\delta)$-satisficing UCL algorithm achieves efficient performance  as guaranteed by the following theorem.
\begin{theorem} \label{thm:MdeltaSatisficing}
Consider the $(\satThresh,\delta)$-satisficing UCL algorithm with an uninformative prior. The number of times the picked arm $i_t$ is non-satisfying with probability greater than $\delta$ is upper bounded as
\[ n_i^T < \frac{4 \sigma_s^2}{\left( \Delta_i^{\satThresh} \right)^2} \left( \Phi^{-1}\left( 1 - \delta/3 \right) \right)^2 + 1. \]
\end{theorem}
\begin{proof}
The proof is very similar to the proofs of Theorems \ref{thm:satisficingInTheMean} and \ref{thm:deltaSufficing}. As in Theorem \ref{thm:satisficingInTheMean}, we bound $n_i^T$ by 
\begin{align*}
n_i^T &= \sum_{t=1}^T \boldsymbol 1(i_t = i) \\
	&\leq \eta \! + \! \sum_{t=1}^T \left(\indicator{Q_i^t \geq \satThresh, n_i^t \geq \eta}
	 \! + \! \indicator{Q_i^t \geq Q_{i^*}^t, n_i^t \geq \eta}\right).
\end{align*}
The condition $Q_i^t \geq \satThresh$, which means arm $i$ is in the eligible set, can be decomposed into the two conditions
\begin{align}
\mu_i^t & \geq m_i + C_i^t  \label{eq:MdeltaCondition1} \\
\satThresh & < m_i + 2 C_i^t. \label{eq:MdeltaCondition2}
\end{align}
 \eqref{eq:MdeltaCondition2} is equivalent to
\[ \Delta_i^{\satThresh} = \satThresh - m_i < 2 C_i^t = \frac{2 \sigma_s}{\sqrt{n_i^t}} \Phi^{-1}(1-\delta/3). \]
Squaring and rearranging, we see that  \eqref{eq:MdeltaCondition2} never holds if 
\begin{align*} n_i^t 
 &> \frac{4 \sigma_s^2}{\left(\Delta_i^{\satThresh} \right)^2} \left( \Phi^{-1}(1-\delta/3) \right)^2 = \eta.
\end{align*}
The same argument as in the proof of Theorem \ref{thm:deltaSufficing} shows that for $n_i^t \geq 1$,  \eqref{eq:MdeltaCondition1} holds with probability at most $\delta/3$, so $n_i^t > \eta$ implies that a non-satisfying arm is in the eligible set with probability at most $\delta/3$.

As in the proof of Theorem \ref{thm:deltaSufficing}, a non-satisfying arm $i$ is picked due to the eligible set being empty only if $Q_i^t \geq Q_{i^*}^t$, where $i^*$ is the arm with maximal mean reward. This condition can again be decomposed into the three conditions \eqref{eq:muitHighDelta}--\eqref{eq:gapDelta}.  \eqref{eq:gapDelta} does not hold if $n_i^t > \eta$, so the probability that $Q_i^t \geq Q_{i^*}^t$ is bounded by the probability that either \eqref{eq:muitHighDelta} or \eqref{eq:muiStartLowDelta} holds. For $n_i^t > 1$, each of these  holds with probability $\delta/3$, so the probability of a non-satisfying arm being chosen due to the eligible set being empty is at most $2\delta/3$. Thus, for $n_i^t > \eta + 1$, a non-satisfying arm is selected with probability at most $\delta$.
\end{proof}

Theorem \ref{thm:MdeltaSatisficing} guarantees that the $(\satThresh,\delta)$-satisficing UCL algorithm achieves finite regret. Furthermore, the algorithm is efficient in that the regret matches the dependence on $\epsilon$ and $\delta$ in the bound \eqref{eq:MdeltaSufficingRegretBound}. Applying the bound \eqref{eq:QDown} on the inverse Gaussian cdf to the statement in the theorem, we see that $n_i^T$ is upper bounded by $8 \sigma_s^2 \log(3/\delta)/\left(\Delta_i^{\satThresh}\right)^2$. Summing this bound over non-satisfying arms $i$ shows that the total number of times the algorithm incurs regret is at most $8 \sigma_s^2 \log(3/\delta) \sum_{\{ i | \Delta_i^{\satThresh}> 0 \}} 1/\left( \Delta_i^{\satThresh} \right)^2$. This matches the dependence on $\epsilon$ and $\delta$ in the bound \eqref{eq:MdeltaSufficingRegretBound} up to constant factors.

Note that lower bound \eqref{eq:MdeltaSufficingRegretBound} counts the number of selections of all arms including the optimal arm, while the upper bound counts only the suboptimal arms. Hence, we can only claim that we achieve cumulative regret bounded in $T$. With a better lower bound on $n_i^T$, we may be able to claim that, similar to $\delta$-sufficing UCL, $(\satThresh, \delta)$-sufficing UCL achieves the optimal dependence on $\epsilon$ and $\delta$. However, this remains an open problem to pursue.

\subsection{Robust satisficing UCL algorithms}
The UCL algorithm solves Problem 1, the Gaussian standard problem. The  modified versions of the UCL algorithm in Sections~\ref{prob2alg}, \ref{prob3alg}, and \ref{prob4alg}  solve the other three Gaussian satisficing-in-mean-reward Problems 2--4. All four UCL algorithms achieve efficient performance in solving their respective problems, as guaranteed by Theorems~\ref{thm:satisficingInTheMean}, \ref{thm:deltaSufficing}, and \ref{thm:MdeltaSatisficing}.

The equivalence result of Lemma \ref{lem:equivalence} shows for Gaussian distributed rewards that  we can modify the four UCL algorithms developed for Problems 1--4 to solve Problems 5--8 as follows. The modified UCL algorithms make decisions based on the standardized mean reward \eqref{eq:defX} using priors on the standardized mean rewards. A prior belief $\bfm \sim \mcN(\bfmu_0,\Sigma_0)$ on the mean rewards $\bfm$ is transformed into a prior belief on the standardized mean rewards $\bfx \sim \mcN(\tilde{\bfmu}_0,\tilde{\Sigma}_0)$ by 
\[(\tilde{\bfmu}_0)_i = ((\bfmu_0)_i-M)/\sigma_{s,i}, \ (\tilde{\Sigma}_0)_{ij} = (\Sigma_0)_{ij}/(\sigma_{s,i} \sigma_{s,j}). \]

\subsubsection{Problem 5: Robust UCL algorithm}
The \emph{robust UCL algorithm} is the UCL algorithm where the prior is given in terms of the standardized mean rewards, and the observed reward $r_t$ is standardized according to the transformation \eqref{eq:standardize} before being input to the inference equations \eqref{eq:inference}.

\subsubsection{Problem 6: Robust satisfaction UCL algorithm}
The \emph{robust satisfaction UCL algorithm} is the satisfaction-in-mean-reward UCL algorithm where the prior is given in terms of the standardized mean rewards, the observed reward $r_t$ is standardized according to the transformation \eqref{eq:standardize} before being input to the inference equations \eqref{eq:inference}, and the parameter $\satThresh$ is set equal to $X = \Phi^{-1}(\Pi)$ defined in \eqref{eq:thresholdX}.

\subsubsection{Problem 7: $\delta$-robust sufficing}
The \emph{$\delta$-robust sufficing UCL algorithm} is the $\delta$-sufficing UCL algorithm where the prior is given in terms of the standardized mean rewards, and the observed reward $r_t$ is standardized according to the transformation \eqref{eq:standardize} before being input to the inference equations \eqref{eq:inference}.

\subsubsection{Problem 8: $(\Pi, \delta)$-robust sufficing}
The \emph{$(\Pi, \delta)$-robust sufficing UCL algorithm} is the $(\satThresh, \delta)$-satisficing UCL algorithm where the prior is given in terms of the standardized mean rewards, the observed reward $r_t$ is standardized according to the transformation \eqref{eq:standardize} before being input to the inference equations \eqref{eq:inference}, and the parameter $\satThresh$ is set equal to $X = \Phi^{-1}(\Pi)$ defined in \eqref{eq:thresholdX}.

Lemma \ref{lem:equivalence} implies that the performance guarantees that hold for the UCL algorithms developed for Problems 1--4 also hold for the four new UCL algorithms defined above when applied to Problems 5--8.

\subsection{Relaxations of  Gaussian and known variance assumptions}
The algorithms presented so far have been developed assuming that the reward distribution associated with each arm $i$ is Gaussian with unknown mean $m_i$ and known variance $\sigma_{s,i}^2$. The reward variance may be known, e.g., estimated from known sensor characteristics or prior data. When the reward variance is not known, a simple modification to the heuristic \eqref{eq:Q} yields an algorithm that achieves efficient performance. Similar simple modifications extend our results to the case where the reward distribution is sub-Gaussian, which includes distributions with bounded support. We state modifications for the case of an uninformative prior. Prior information can be incorporated using a conjugate prior, as discussed in \cite{KPM:07}.

\begin{remark}[Gaussian rewards with unknown variance]
\label{remark-unknown-variance}
When the reward distribution is Gaussian with unknown variance, the heuristic developed by Auer~\etal~\cite{PA-NCB-PF:02} for their algorithm UCB1-NORMAL results in algorithms that achieve efficient performance. Recall that $n_i^t$ is the number of times arm $i$ has been selected up to time $t$, and $\bar{m}_i^t$ is the empirical mean reward observed at arm $i$ up to time $t$. Define
\[ q_i^t = \sum_{t | i_t = i} r_t^2 \]
as the sum of the squared rewards obtained from arm $i$.

The UCB1-NORMAL algorithm is composed of two rules: if there is an arm that has been played less than $\lceil 8 \log t \rceil$ times, it selects that arm. Otherwise it selects the arm $i$ that maximizes the heuristic
\[ Q_i^{t, \mathrm{UCB1-NORMAL}} = \bar{m}_i^t + \sqrt{16 \frac{q_i^t - n_i^t (\bar m_i^t)^2}{n_i^t-1} \frac{\log t}{n_i^t}}. \]
This heuristic can be used directly in the standard  and satisfaction-in-mean-reward UCL algorithms. For the $\delta$-sufficing and $(\satThresh, \delta)$-satisficing UCL algorithms, use $\tilde{k} = 2$ and $\tilde{k} = 3$, respectively, in the heuristic
\begin{align*} Q_i^{t} &= \bar{m}_i^t + \sqrt{4 \frac{q_i^t - n_i^t (\bar m_i^t)^2}{n_i^t-1} \frac{\log (\tilde{k}/\delta)}{n_i^t}}.
\end{align*}

The Gaussian distribution with unknown mean and variance is again a location-scale family, so Lemma \ref{lem:equivalence} implies that these modified algorithms can be used to solve the robust satisficing problems as well.

Prior information can be incorporated by means of a conjugate prior, as discussed in \cite{KPM:07}.
\end{remark}

\begin{remark}[Sub-Gaussian rewards]
Another generalization of Gaussian rewards with known variance is the case where the reward distribution is sub-Gaussian, also known as light-tailed. The distribution of a random variable $X$ is called sub-Gaussian if its moment generating function $M(u) = \E{\exp(u X)}$ is finite for all $u \in \bbR$. Then, one can find a constant $\zeta$ such that $M(u) \leq \exp(\zeta u^2/2)$ \cite{PC-OC-SK:06}.

In this case, a heuristic function due to Liu and Zhou \cite{KL-QZ:11}
\[ Q_i^{t, \mathrm{SG}} = \bar{m}_i^t + \sqrt{\frac{8 \zeta \log t}{n_i^t}} \]
 can be used to achieve efficient performance.
\end{remark}

\begin{remark}[Reward distributions with bounded support]
Another common assumption in the bandit literature is that the reward distributions are arbitrary but have a known bounded support $[a, b] \subset \bbR$. Without loss of generality, we assume that the support is contained in the unit interval $[0,1]$. In this case the UCB1 heuristic due to Auer \etal~\cite{PA-NCB-PF:02}
\[ Q_i^{t, \mathrm{UCB1}} = \bar{m}_i^t + \sqrt{\frac{2 \log t}{n_i^t}} \]
can be used in the standard and satisfaction-in-mean-reward UCL algorithms.

For the $\delta$-sufficing and $(\satThresh, \delta)$-satisficing UCL algorithms,  use $\tilde k = 2$ and $\tilde k = 3$, respectively, in the heuristic
\begin{align*} Q_i^{t} &= \bar{m}_i^t + \sqrt{ \frac{\log(\tilde k/\delta)}{2 n_i^t} }. \end{align*}

For the robust satisficing problems the relevant reward, happiness $h_t$ \eqref{eq:happiness}, is a Bernoulli random variable which is supported on $[0,1]$. Therefore, each robust satisficing problem can be solved by the appropriate variant of UCB1. However, if additional information is available about the distribution of the raw rewards $r_t$, e.g., that they are Gaussian with known variance, then the robust UCL algorithms can achieve improved performance relative to UCB1, for example if the Kullback-Leibler divergence between the $r_t$ distributions is larger than the Bernoulli distributions associated with $h_t$.

Additional extensions to heavy-tailed distributions may be possible using the techniques of \cite{Bubeck2013}.
\end{remark}

\section{Numerical examples}
In this section, we present the results of numerical simulations of the modified UCL algorithms solving multi-armed bandit problems with Gaussian rewards and satisficing objectives.  We consider both thresholding in the mean rewards $m_i$, as in Problems 1--4 (Table \ref{tab:thresholdm}), and thresholding in the instantaneous rewards, as in Problems 5--8 (Table \ref{tab:thresholdr}). In all of the cases presented, the algorithms used an uninformative prior.   We use the simulations to illustrate performance of the algorithms relative to the bounds proved in the theorems of Section IV.   We also use the simulations to compare how the different algorithms trade off accumulation of reward with reduction in exploration cost as measured by number of switches among arms. As shown in the figures, satisficing can significantly decrease the exploration cost while incurring little cost in terms of the rewards received by the agent.

We first consider the satisficing objectives with thresholding in the mean rewards.  We illustrate how the objectives of Problems 1 and 2 yield logarithmic regret  (Figure \ref{fig:satisficingWithCertainty}) whereas the objectives of Problems 3 and 4 yield finite regret (Figure \ref{fig:sufficingInTheMean}), as predicted by the bounds proved in Theorems \ref{thm:UCL}, \ref{thm:satisficingInTheMean}, \ref{thm:deltaSufficing} and \ref{thm:MdeltaSatisficing}. For the simulations presented in Figures \ref{fig:satisficingWithCertainty} and \ref{fig:sufficingInTheMean}, we set $N=4$. The mean rewards $\bfm$ were set equal to $[1\ 2\ 3\ 4]$ and the standard deviations $\sigma_{s,i}$ were each set equal to $1$.

In Figure \ref{fig:satisficingWithCertainty}, the agent's regret is defined  by comparing the mean rewards $m_i$ with the satisfaction threshold $\satThresh$. For the standard objective (Problem 1) the satisfaction level $\satThresh$ was set equal to $m_{i^*} = 4$, so the agent incurred regret if it selected any arm other than $i^* = 4$. For the satisfaction-in-mean-reward objective (Problem 2) the satisfaction level $\satThresh$ was set equal to 2.5, so the agent  incurred regret if it selected arms 1 or 2. Figure \ref{fig:satisficingWithCertainty} plots mean cumulative regret from 100 simulations (solid lines) and the bounds on regret from Theorems \ref{thm:UCL} and \ref{thm:satisficingInTheMean} (dashed lines) for the standard and satisfaction-in-mean-reward UCL algorithms, respectively. We observe from Figure \ref{fig:satisficingWithCertainty} that the algorithms' regret is significantly below the bounds, indicating that both bounds are conservative. Because both objectives set the sufficiency threshold $\delta=0$ and define regret in terms of the unknown mean rewards $m_i$, the agent must achieve certainty about the mean reward values to stop incurring regret. It is impossible for the agent to achieve this certainty in finite time, so the mean regret and its bound both increase indefinitely at a logarithmic rate for both objectives.

Figure \ref{fig:sufficingInTheMean}, in contrast, shows that by setting the sufficiency threshold $\delta$ to a non-zero value, one can achieve finite regret. All the parameters for the simulations shown in Figure \ref{fig:sufficingInTheMean} were identical to those for the simulations in Figure \ref{fig:satisficingWithCertainty}, except that the sufficiency threshold $\delta$ was set equal to 0.05. Setting $\delta$ to a non-zero value transformed the standard and satisfaction-in-mean-reward objectives into the $\delta$-sufficing and $(\satThresh, \delta)$-satisficing objectives, Problems 3 and 4, respectively. For these objectives, regret is again defined by thresholding the  mean rewards $m_i$. However, rather than seeking certainty that the threshold is met, the agent only seeks to ensure that its threshold is met with a probability of at least $1-\delta$. Because of the allowance of uncertainty, the agent only needs to perform a finite amount of exploration before settling on an arm that appears satisfying. The regret bounds in the figure follow from Theorems \ref{thm:deltaSufficing} and \ref{thm:MdeltaSatisficing} for the $\delta$-sufficing and $(\satThresh, \delta)$-satisficing objectives, respectively. As in Figure \ref{fig:satisficingWithCertainty}, we observe from Figure \ref{fig:sufficingInTheMean} 
that both bounds are conservative.

The reduction in exploring that comes with sufficing can be advantageous when exploring is costly, e.g., when there is a cost associated with making a switch from one arm to another.  Figure \ref{fig:explorationSavings} suggests that this reduction in cost may require little sacrifice in reward. The upper panel plots mean cumulative reward from 100 simulations for both the standard bandit problem 1 and the $\delta$-sufficing problem 3. The curve for $\delta$-sufficing is slightly below that for the standard bandit problem, showing that it results in slightly lower cumulative rewards, but the difference is insignificant in comparison to the overall magnitude of the cumulative rewards. The lower panel plots the mean cumulative number of switches between arms for both algorithms and shows that $\delta$-sufficing requires roughly half as much exploration.

\begin{figure}[ht]
\centering
\includegraphics[trim=3mm 2mm 3mm 13mm, clip=true, width=3.6in]{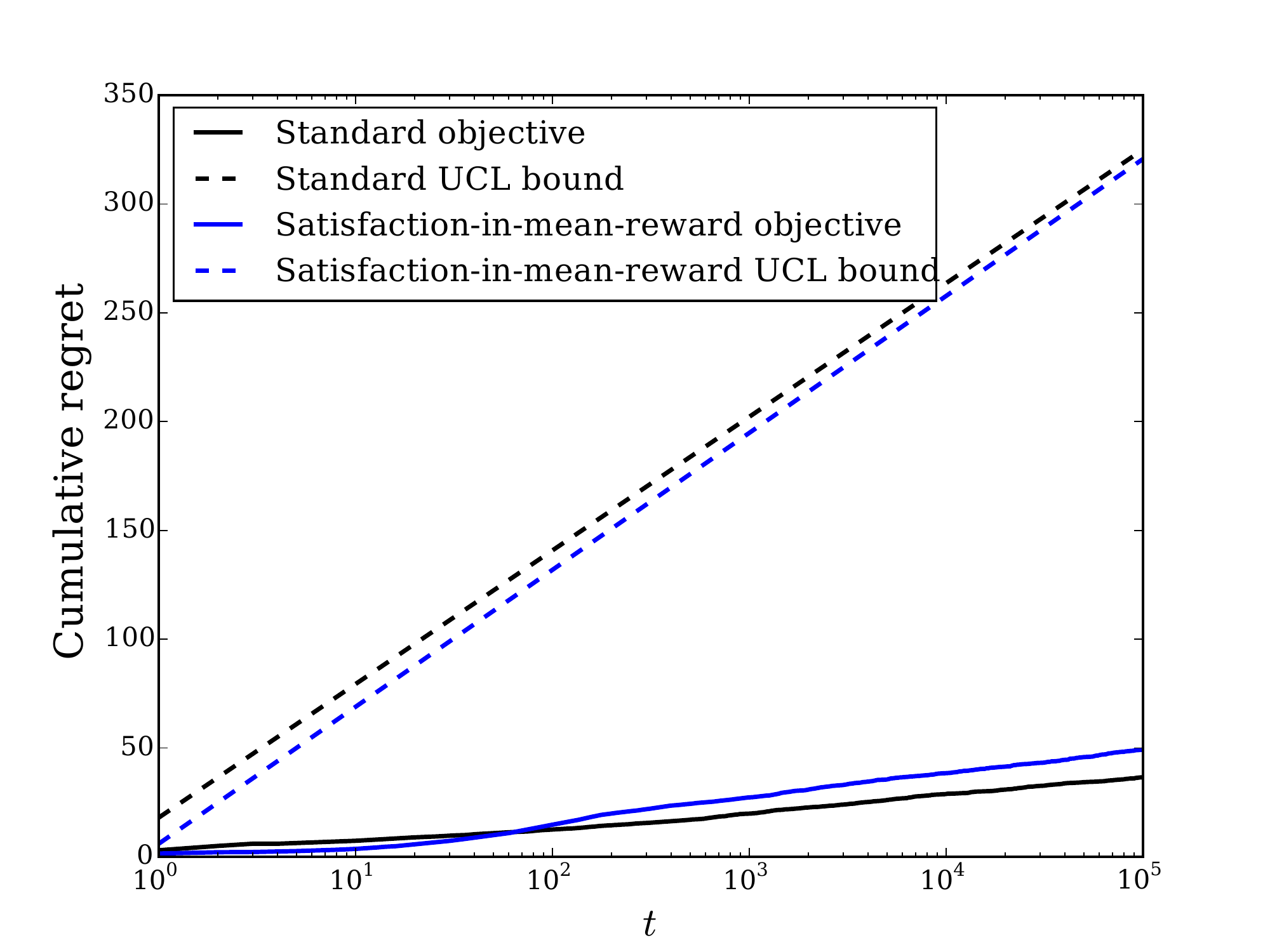}
\caption{Comparison of regret incurred by the UCL algorithms when solving the standard  problem (Problem 1) and  satisfaction-in-mean-reward problem (Problem 2). Both problems define regret by thresholding  mean reward values; the standard bandit objective incurs regret when the mean reward of the chosen option is less than the maximum reward $m_{i^*}$, while the satisfaction-in-mean-reward problem incurs regret when the mean reward is less than $\satThresh \le m_{\sigma(2)}$, here set equal to 2.5. For both problems, the cumulative expected regret and its upper bound increase at a logarithmic rate since the agent seeks certainty that its threshold is met, which it cannot achieve in finite time.}
\label{fig:satisficingWithCertainty}
\end{figure}

\begin{figure}[ht]
\centering
\includegraphics[trim=3mm 2mm 3mm 13mm, clip=true, width=3.6in]{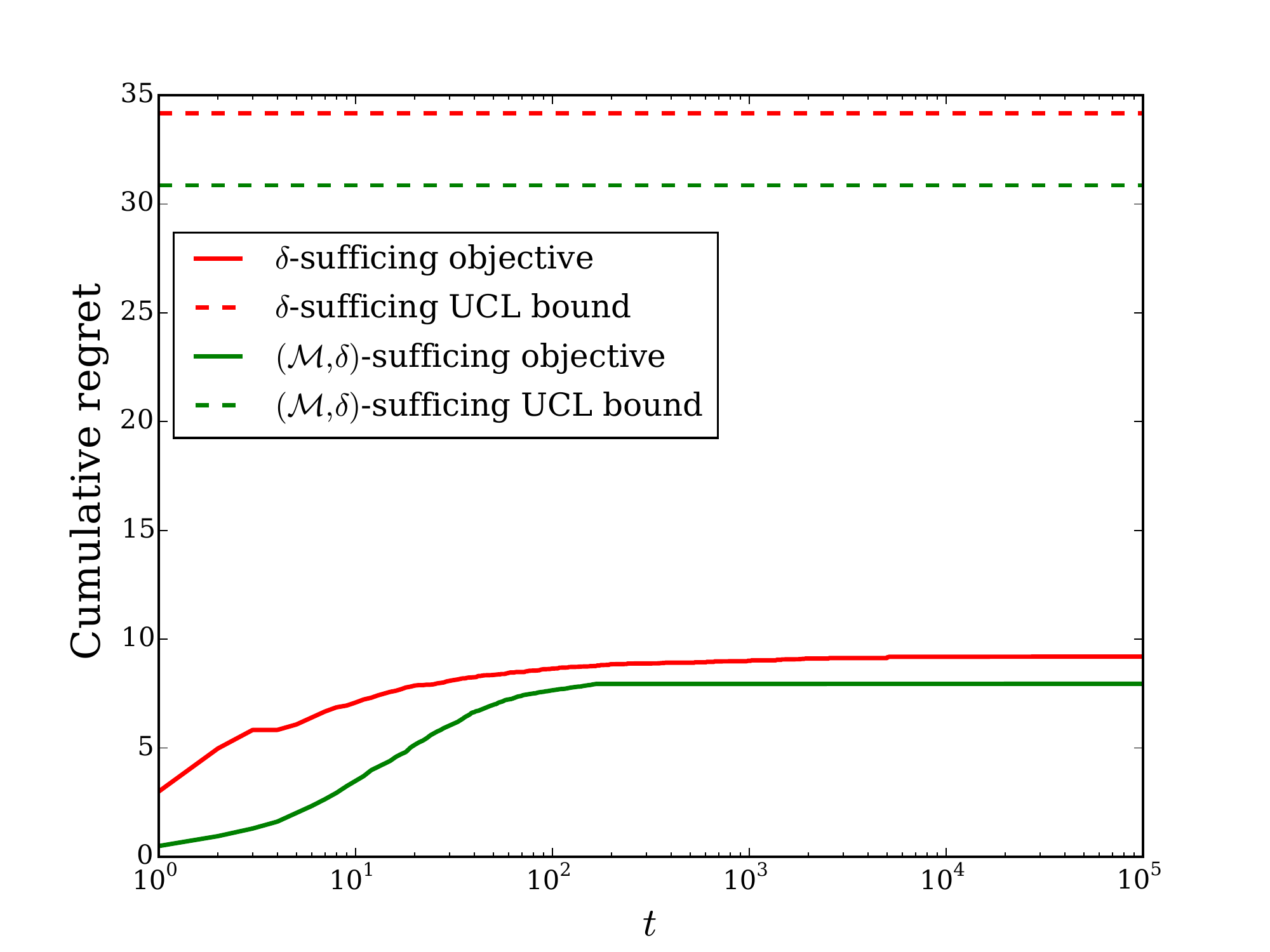}
\caption{Comparison of regret incurred by the UCL algorithms when solving the $\delta$- and $(\satThresh,\delta)$-satisficing problems, (Problems 3 and 4, respectively). As in Figure \ref{fig:satisficingWithCertainty}, the problems define regret by thresholding the mean reward values; the $\delta$-sufficing objective incurs regret when the mean reward of the chosen option  is less than the maximum reward $m_{i^*}$, while the $(\satThresh,\delta)$-sufficing problem incurs regret when the mean reward is less than $\satThresh \le m_{\sigma(2)}$, here set equal to 2.5. In contrast to Figure \ref{fig:satisficingWithCertainty},  the agent only seeks to have $1-\delta = 95\%$ confidence that its threshold is met, which it can achieve in finite time. Thus, the upper bounds on cumulative expected regret are constant functions of horizon length and the mean regret  plateaus at a finite value.}
\label{fig:sufficingInTheMean}
\end{figure}

We next consider the satisficing objectives with thresholding in the instantaneous rewards. Figure \ref{fig:satisficingAtEachTime} presents a simulation that demonstrates the equivalence result of Lemma \ref{lem:equivalence} for the robust bandit  (Problem 5)  using the robust UCL algorithm. The happiness threshold $M$ was set equal to 2. As in Figures \ref{fig:satisficingWithCertainty} and \ref{fig:sufficingInTheMean}, the mean rewards $\bfm$ were set equal to $[1\ 2\ 3\ 4]$, but for this simulation the standard deviations were set equal to $[1\ 1\ 1\ 3]$.   So the standardized mean rewards were $\bfx = [-1\ 0\ 1\ \tfrac{2}{3}]$ and $i^* = 3$ the optimal arm, i.e., the arm with maximal happiness probability. Figure \ref{fig:satisficingAtEachTime} shows mean cumulative regret from 100 simulations (solid line) and the regret bound (dashed line) implied by Theorem \ref{thm:UCL} and the definition \eqref{eq:satisfactionR} of satisfaction. Because the objective requires identifying the arm with highest probability of satisfaction with certainty, both the regret and its upper bound increase indefinitely at a logarithmic rate, as in the problems illustrated in Figure \ref{fig:satisficingWithCertainty}.

\vspace{-1mm}
\begin{figure}[ht]
\centering
\includegraphics[trim=3mm 2mm 3mm 13mm, clip=true, width=3.6in]{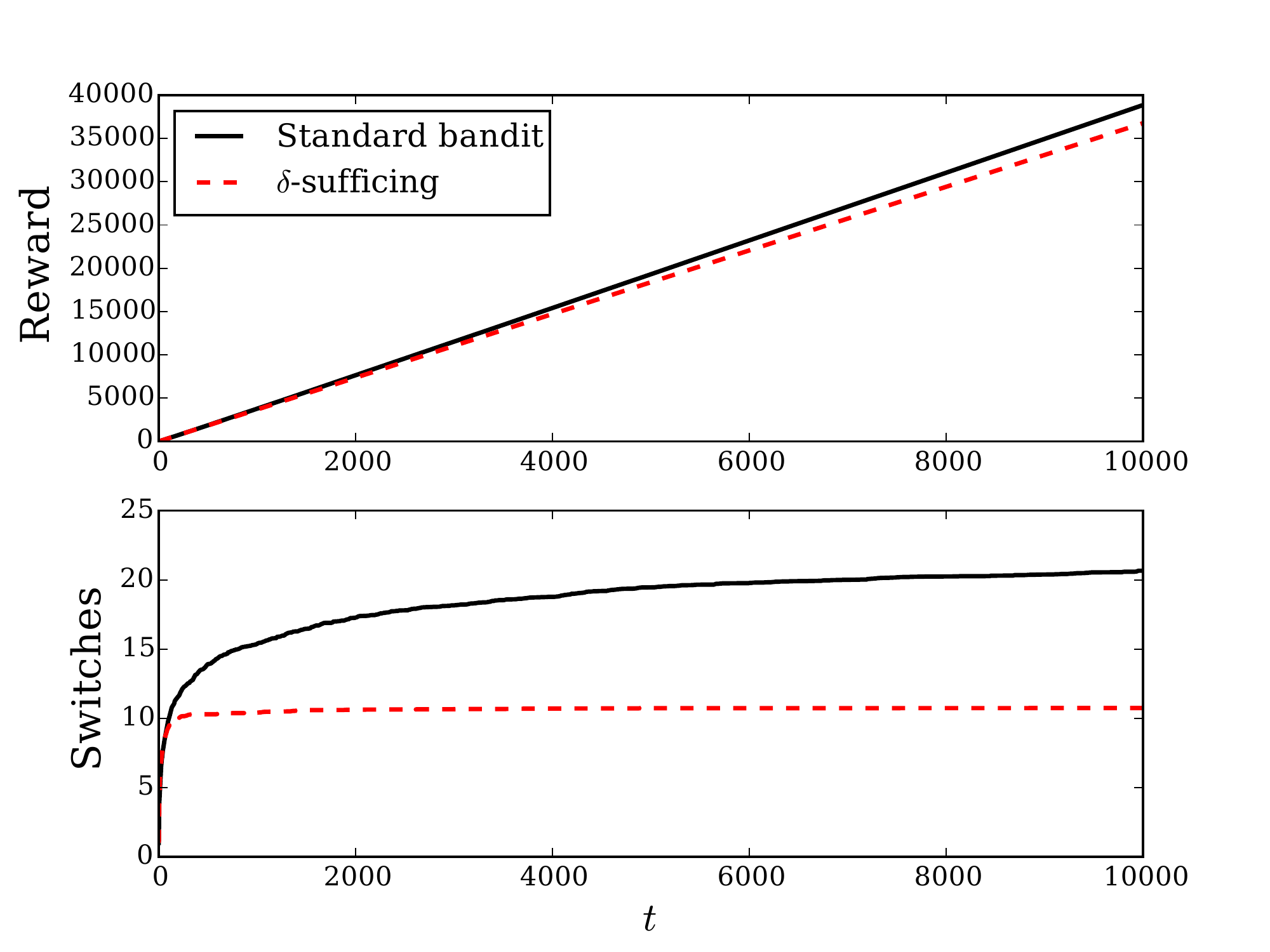}
\caption{Upper: Mean cumulative reward accrued by  standard  (Problem 1) and $\delta$-sufficing  (Problem 3) UCL algorithms. Lower: Mean cumulative number of switches between arms, quantifying the algorithms' exploration costs. The $\delta$-sufficing UCL algorithm achieves nearly the same cumulative rewards as the standard UCL algorithm, but with roughly half the exploration cost.}
\vspace{-2.5mm}
\label{fig:explorationSavings}
\end{figure}

\begin{figure}[ht]
\centering
\includegraphics[trim=3mm 2mm 3mm 13mm, clip=true, width=3.6in]{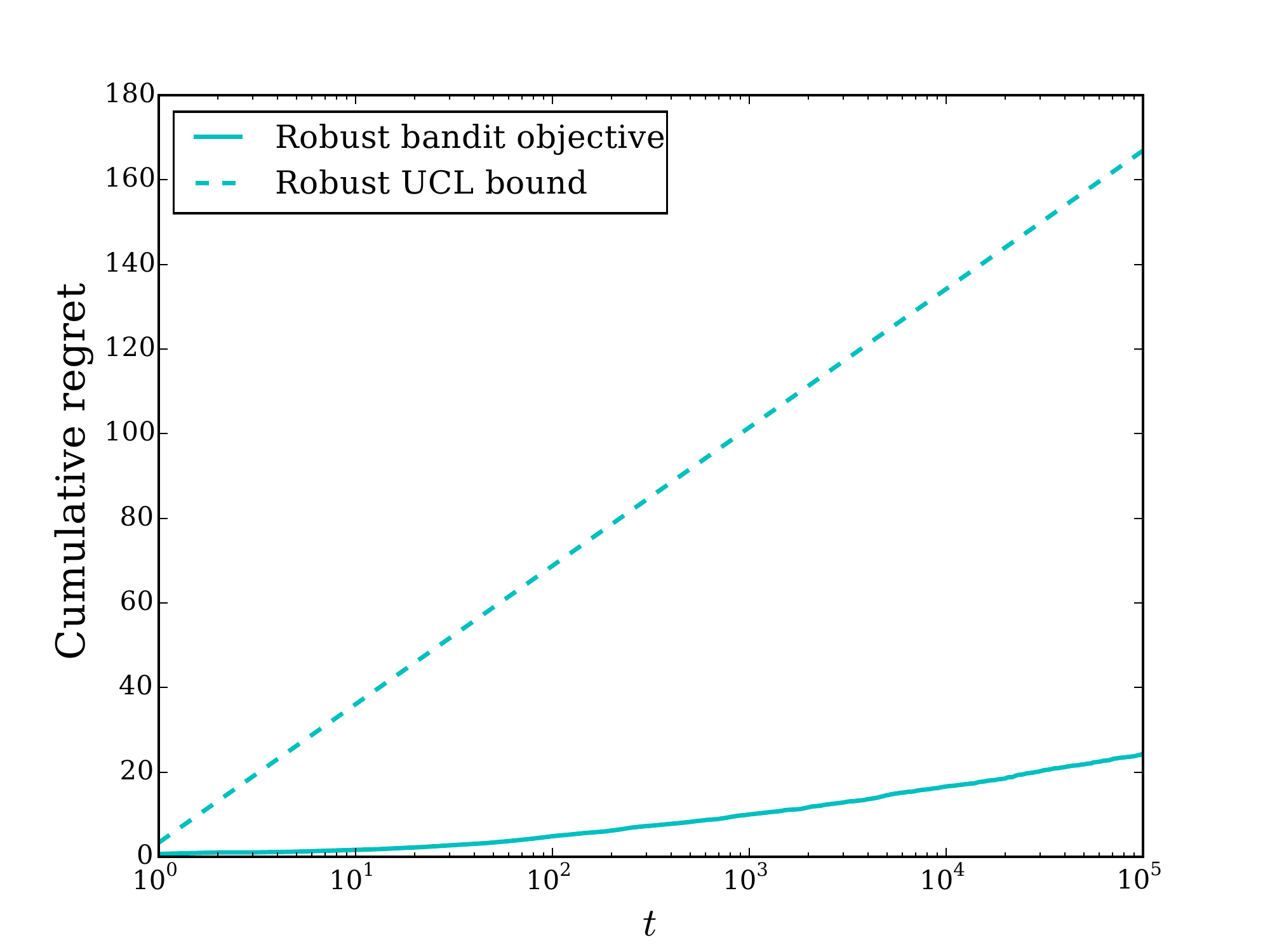}
\caption{Regret incurred by the robust UCL algorithm  (Problem 5). The agent is happy if it receives a reward $r_t$ that is at least equal to $M$, here set equal to 2. The agent seeks to maximize its probability of being happy at each time, so it incurs regret if it chooses an option with less than maximal happiness probability. As in Figure \ref{fig:satisficingWithCertainty}, the agent seeks certainty that it maximizes its happiness probability, which it cannot achieve in finite time.  So it incurs regret that increases indefinitely at a logarithmic rate.}
\vspace{-2.5mm}
\label{fig:satisficingAtEachTime}
\end{figure}

Figure \ref{fig:gainsAndSavings} shows the benefit of combining the robust objective with sufficing in the $\delta$-robust bandit objective (Problem 7). By doing so, it is possible to retain the robustness benefit of the robust bandit objective (Problem 5) relative to the standard bandit objective (Problem 1) while reducing the exploration cost. The parameter values used in the $\delta$-robust UCL algorithm (Problem 7) are the same as in Figure \ref{fig:satisficingAtEachTime}, while the sufficiency parameter $\delta$ was set equal to 0.05. Figure \ref{fig:gainsAndSavings} shows how robust satisficing algorithms (Problems 5 and 7) outperform the standard algorithm (Problem 1) for performance with respect to instantaneous reward as measured by happiness \eqref{eq:happiness}. There is a switching cost associated with achieving the higher rate of cumulative happiness. However, this cost is significantly reduced for the $\delta$-robust algorithm (Problem 7), where sufficing is included, as compared to the robust algorithm (Problem 5) and it approaches the switching cost incurred by the standard UCL algorithm.

\begin{figure}[ht]
\centering
\includegraphics[trim=3mm 2mm 3mm 13mm, clip=true, width=3.6in]{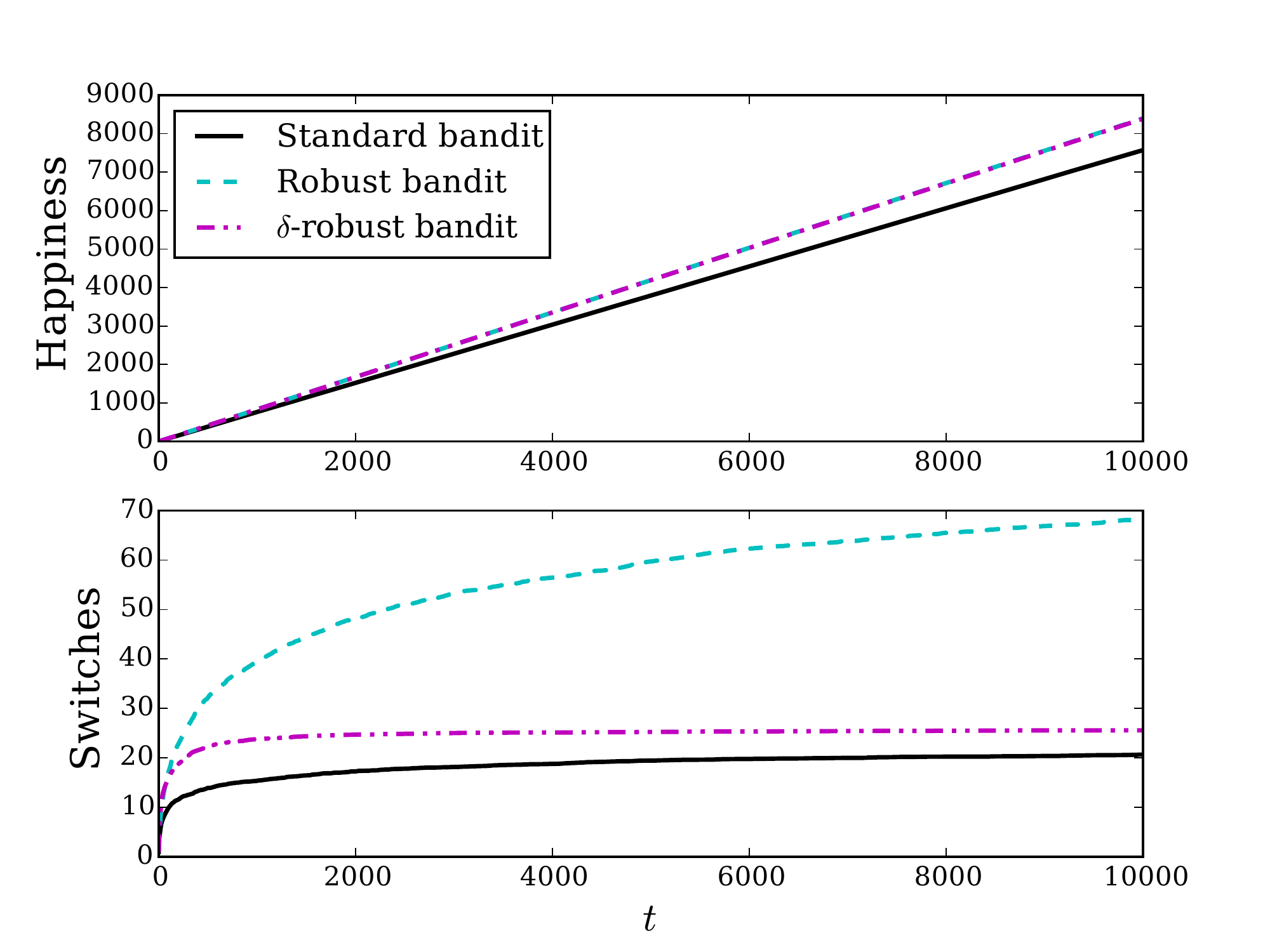}
\caption{Upper: Mean cumulative number of times the agent is happy  (i.e., has, reward $r_t \geq M = 2$, \eqref{eq:happiness}) when using the standard  (Problem 1), robust UCL (Problem 5), and $\delta$-robust  (Problem 7) UCL algorithms. Lower: Mean cumulative number of switches between arms, quantifying the algorithms' exploration costs. The robust bandit objective is more robust in the sense that it is more likely to achieve a reward that is above the threshold $M$, but in this case incurs a greater exploration cost. The robust $\delta$-sufficing objective combines the positive aspects of both the $\delta$-sufficing and robust bandit objectives: exhibiting high robustness in terms of agent happiness and minimizing exploration costs.}
\vspace{-5mm}
\label{fig:gainsAndSavings}
\end{figure}

\section{Conclusion}
Satisficing, the concept of doing well relative to a reference value, is a useful alternative to maximizing that can be applied to a variety of decision-making scenarios. In this paper, we considered the multi-armed bandit problem using satisficing objectives. The multi-armed bandit problem is a canonical decision-making problem that is widely studied in machine learning and adaptive control using a maximization objective.

We proposed a system of eight objectives for stochastic multi-armed bandit problems that generalize the standard multi-armed bandit problem by capturing aspects of satisficing, notably thresholding effects that we termed \emph{satisfaction} and \emph{sufficiency}. We showed that each of the four objectives of Problems 1--4, defined by thresholding the unknown mean reward $m_i$ associated with each arm, is equivalent to a related problem studied in the existing literature. We used these equivalences to derive bounds on efficient performance. For the four objectives of Problems 5--8 defined by thresholding the observed rewards $r_t$, we showed that, when the reward distributions belong to a location-scale family, each objective is equivalent to one of the first four objectives.

We then specialized to the case of Gaussian rewards (a particular location-scale family) and developed four variants of the UCL algorithm \cite{PR-VS-NEL:14} to solve Problems 1--4 defined by thresholding the mean rewards. We analyzed each algorithm and showed that it achieved efficient satisficing performance. We used the equivalency result (Lemma \ref{lem:equivalence}) to show how to apply the four variants of the UCL algorithm to Problems 5--8 defined by thresholding the instantaneous rewards and again achieve efficient satisficing performance.

Satisficing objectives that threshold the mean can reduce exploration costs (and thus risk) as compared to the standard problem where the objective is to maximize expected reward with certainty. Satisficing objectives that threshold  observed rewards can result in more risk-averse and robust algorithms than objectives that account only for mean rewards \cite{PR-NEL:14a}. Risk aversion and robustness are important for engineering applications (where standard bandit algorithms are known to have poor risk-aversion characteristics~\cite{JYA-RM-CS:09}). Thus, our proposed algorithms can be usefully applied to a range of engineering problems, notably those involving design of control policies.

Risk aversion and robustness are also important in the field of optimal foraging theory~\cite{YC-YBH:05}. Foraging has been studied using the multi-armed bandit framework with a maximizing objective~\cite{JRK-AK-PT:78, TK-etal:02, VS-PR-NEL:13}. The satisficing objectives and algorithms that we have proposed in the present paper may provide an even more biologically plausible framework.

For any satisficing problem, selecting the appropriate satisfaction threshold remains an open problem. The results presented here provide an efficient policy once the satisfaction threshold has been chosen but leave the selection of the threshold up to the end user. The problem of choosing a threshold also arises when applying the Sequential Probability Ratio Test (SPRT) in hypothesis testing \cite{AW:45}. To apply the SPRT, one must select desired probabilities of type I and type II errors. Once these probabilities are selected the SPRT provides the optimal policy, but the SPRT itself does not provide optimal values for the error probabilities.

In many decision-making scenarios, maximizing a reward rate is used to optimize  error probabilities \cite{RB-etal:06}. A reward rate criterion may provide an optimal threshold for the satisficing algorithms developed here, but the specific criterion will depend on the decision-making scenario. Natural extensions also include considering cases where the mean rewards $m_i$ are allowed to evolve over time, for example according to a jump process \cite{VS-PR-NEL:14}; such evolution will likely encourage satisficing policies that incorporate adaptive thresholds. Finally, it is well understood that satisficing is an important feature of human decision making \cite{BS-etal:02} and that the UCL algorithm can model many features of human decision decision making in bandit tasks \cite{PR-VS-NEL:14}. New empirical work should be undertaken to compare the satisficing UCL algorithms with human behavior. It is an open empirical question to determine which of our notions of regret best explains human behavior.

\section*{Acknowledgement}
We thank Simon A. Levin for helpful discussions, as well as the three anonymous referees and the editor for their feedback which greatly strengthened the paper.

\bibliographystyle{abbrv}
\bibliography{satisficing}

\small{\textbf{Paul Reverdy} (M '14) received the B.S. degree in engineering physics and the B.A. degree in applied mathematics from the University of California, Berkeley, Berkeley, CA, USA, in 2007 and the M.A. and Ph.D degrees in mechanical and aerospace engineering from Princeton University, Princeton, NJ, USA, in 2011 and 2014, respectively.

From 2007 to 2009, he worked as a Research Assistant at the Federal Reserve Board of Governors, Washington, DC, USA. He is currently a Postdoctoral Fellow with the Department of Electrical and Systems Engineering, University of Pennsylvania, Philadelphia, PA, USA. His research interests are in the areas of control and robotics with current interests in human and automated decision making, engineering design, and navigation.}

\small{\textbf{Vaibhav Srivastava} received the B.Tech. degree (2007) in mechanical engineering from the Indian Institute of Technology Bombay, Mumbai, India; the M.S. degree in mechanical engineering (2011), the M.A. degree in statistics (2012), and and the Ph.D. degree in mechanical engineering (2012) from the University of California at Santa Barbara, Santa Barbara, CA. He served as a Lecturer and Associate Research Scholar with the Mechanical and Aerospace Engineering Department, Princeton University, Princeton, NJ from 2013-2016. 

Srivastava is an Assistant Professor of Electrical and Computer Engineering at Michigan State University. His research interests include modeling and analysis of human cognition; shared autonomous systems; socio-cognitive networks; computational networks; and robotic search and surveillance problems.}

\small{\textbf{Naomi Ehrich Leonard} (F '07)
received the B.S.E. degree in mechanical engineering from Princeton University, Princeton, NJ, in 1985 and the M.S. and Ph.D. degrees in electrical engineering from the University of Maryland, College Park, in 1991 and 1994, respectively.  From 1985 to 1989, she worked as an Engineer in the electric power industry.  

Leonard is the Edwin S. Wilsey Professor of Mechanical and Aerospace Engineering and Director of the Council on Science and Technology at Princeton University.  She is also an associated faculty member of Princeton University's Program in Applied and Computational Mathematics.  Leonard's research and teaching are in control and dynamical systems with current interests in coordinated control for multi-agent systems, mobile sensor networks, collective animal behavior, and human decision-making dynamics.}

\newpage

\end{document}